\documentclass[preprint,12pt,authoryear]{elsarticle}

\usepackage{amssymb}
\usepackage{amsmath}
\usepackage{amsthm}
\usepackage{thm-restate}

\newcommand{\model}{LieSD}
\usepackage{xcolor}
\usepackage{booktabs}
\usepackage{multirow}
\usepackage{subcaption}

\journal{Neural Networks}

\begin{document}

\begin{frontmatter}



\title{Symmetry Discovery for Different Data Types} %% Article title


\author[a]{Lexiang Hu}
\ead{hulx@stu.pku.edu.cn}
\author[a]{Yikang Li}
\ead{liyk18@pku.edu.cn}
\author[a,b,c]{Zhouchen Lin\corref{d}}
\ead{zlin@pku.edu.cn}

\address[a]{State Key Lab of General AI, School of Intelligence Science and Technology, Peking University}
\address[b]{Institute for Artificial Intelligence, Peking University}
\address[c]{Pazhou Laboratory (Huangpu), Guangzhou, Guangdong, China}
\cortext[d]{Corresponding author}

%% Abstract
\begin{abstract}
Equivariant neural networks incorporate symmetries into their architecture, achieving higher generalization performance. However, constructing equivariant neural networks typically requires prior knowledge of data types and symmetries, which is difficult to achieve in most tasks. In this paper, we propose \model, a method for discovering symmetries via trained neural networks which approximate the input-output mappings of the tasks. It characterizes equivariance and invariance (a special case of equivariance) of continuous groups using Lie algebra and directly solves the Lie algebra space through the inputs, outputs, and gradients of the trained neural network. Then, we extend the method to make it applicable to multi-channel data and tensor data, respectively. We validate the performance of \model~on tasks with symmetries such as the two-body problem, the moment of inertia matrix prediction, and top quark tagging. Compared with the baseline, \model~can accurately determine the number of Lie algebra bases without the need for expensive group sampling. Furthermore, \model~can perform well on non-uniform datasets, whereas methods based on GANs fail.
\end{abstract}



%% Keywords

\end{frontmatter}



\section{Introduction}
\label{sec:introduction}

Symmetries in different data types play important roles in deep learning. Many recent works embed symmetries into network structures \citep{cohen2016steerable,weiler2018learning,weiler2019general}, not only improving model generalization but also reducing the number of parameters, which significantly enhances training performance on specific tasks. For instance, convolutional neural networks \citep{lecun1998gradient} introduce the translational symmetry in image data, outperforming multilayer perceptrons in computer vision. Group equivariant convolutional networks \citep{cohen2016group} incorporate the rotational symmetry of image data. Equivariant multilayer perceptrons \citep{finzi2021practical} propose a method for constructing networks with more general Lie group symmetries.

However, these methods for constructing equivariant networks all require prior knowledge of symmetries. For example, in the image classification problem, translational and rotational invariance stems from humans' prior knowledge of the task. When faced with complex tasks and messy data, such as physical dynamic systems, it is very difficult to know symmetries in advance, and embedding incorrect symmetries can lead to a significant decrease in network performance.

A series of works on symmetry discovery emerge subsequently \citep{dehmamy2021automatic,romero2022learning,moskalev2022liegg,tegner2023self}. Augerino \citep{benton2020learning} parameterizes group transformations and discovers symmetries in tasks by optimizing these parameters during training. However, parameterization requires a rough understanding of the form of symmetries prior, which limits it to finding sub-group symmetries of given group transformations. Data augmentation can also incur significant computational overhead.

LieGAN \citep{yang2023generative} and LaLiGAN \citep{yang2023latent} directly search for symmetries from the data by bringing the distributions of original and transformed data closer. However, this method demands a high level of symmetry in the distribution of the dataset, which is difficult to achieve in the real world. For example, in facial recognition datasets, most images are frontal faces, which leads to insufficient representation of the rotational symmetry. Additionally, in datasets of particle motion trajectories, particles may also be concentrated in specific regions rather than uniformly distributed throughout the entire space. LieGAN and LaLiGAN also need to specify the coefficient distribution of the Lie algebra bases, since they require group sampling. Furthermore, incorrect setting of the number of Lie algebra bases can lead to them learning incorrect Lie algebra bases.

In this paper, we propose a method based on trained neural networks to discover symmetries, not only invariance but also equivariance. It can identify single-connected Lie group symmetries without relying on uniform datasets. We first propose a theorem to demonstrate that equivariance of the network can be measured solely based on the network's inputs, outputs, gradients, and Lie algebra bases, without the need for data augmentation or group sampling. In practice, we need a trained neural network to approximate the true mapping of the task. Then, based on this theorem, we construct and solve a system of linear equations to obtain orthogonal Lie algebra bases, thereby discovering symmetries. We perform SVD on the coefficient matrix of the linear equations, where the number of zero singular values corresponds to the number of Lie algebra bases. Moreover, we extend our method to the multi-channel and tensor cases, whereas previous symmetry discovery methods only consider the single-channel vector case. The framework of our Lie algebra based symmetry discovery (\model) method is shown in Figure~\ref{fig:introduction}.

\begin{figure}[ht]
	\centering
	\includegraphics[width=\textwidth]{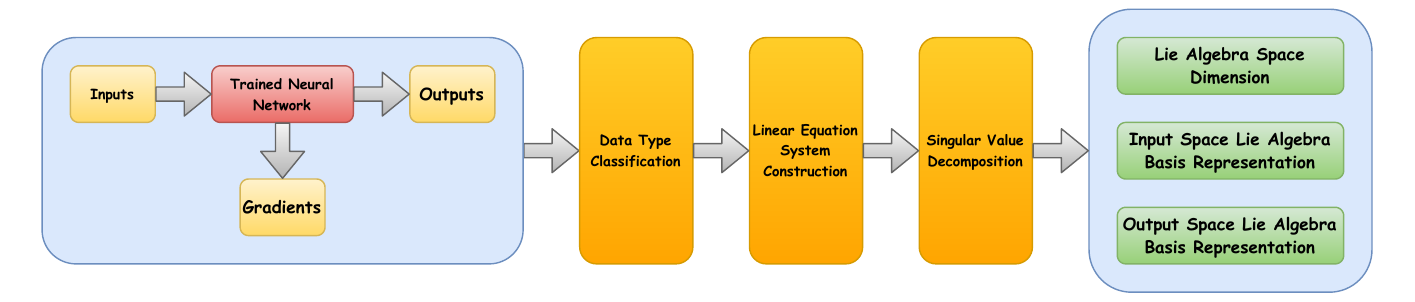}
	\caption{The framework of \model. We first train a neural network on the task to approximate the input-output mapping, and then use its inputs, outputs, and gradients for symmetry discovery. Adapted to different data types, \model~can solve the Lie algebra space, including the dimension of the Lie algebra space and the Lie algebra basis representations.}
	\label{fig:introduction}
\end{figure}

In summary, our contributions are as follows:

\begin{itemize}
	
	\item We provide a theorem that measures equivariance in the neural network using its inputs, outputs, and gradients, without the need for data augmentation or group sampling.
	
	\item We propose \model, a method for solving the Lie algebra space, thus discovering symmetries via trained neural networks. The dimension of the Lie algebra space can be determined by the number of zero singular values.
	
	\item We extend \model~to discover symmetries in multi-channel and tensor data.
	
	\item We verify that \model~can correctly discover symmetries in tasks such as the two-body problem, the moment of inertia matrix prediction, top quark tagging, and outperform the baseline, especially when the datasets are not uniform.
	
\end{itemize}

\section{Related work}
\label{sec:related-work}

\subsection{Equivariant networks}

Many recent works have embedded equivariance into network structures layerwise, which improves their generalization on specific tasks. G-CNNs \citep{cohen2016group} and steerable CNNs \citep{cohen2016steerable} introduce discrete group equivariance into CNNs. SFCNNs \citep{weiler2018learning} and E(2)-Equivariant Steerable CNNs \citep{weiler2019general} extend discrete group equivariance to continuous cases, such as arbitrary angle rotations in group $\mathrm{SO(2)}$. 3D Steerable CNNs \citep{weiler20183d} extends the group action space to $\mathbb{R}^3$. EMLP \citep{finzi2021practical} and affConv/homConv \citep{macdonald2022enabling} respectively embed arbitrary Lie group equivariance into MLP and CNNs. Furthermore, a series of works \citep{shen2020pdo,shen2021pdo,shen2022pdo,he2022neural,li2024affine} construct equivariant networks based on partial differential operators.

\subsection{Symmetry discovery}

Since the construction of equivariant networks requires prior knowledge of equivariance, many subsequent works attempt to discover symmetries. Some works are based on neural networks to detect symmetries. Augerino \citep{benton2020learning} parametrizes group actions. It performs augmentation on input data, then takes the average of the results as the final output of the network. However, it requires strong prior knowledge of group actions, namely the known parameterization of group actions. \citet{krippendorf2020detecting} uses the embedding layer of the neural network to identify orbits of the symmetries in the input. L-conv \citep{dehmamy2021automatic} can automatically discover symmetries and serve as a building block to construct group equivariant feedforward architecture. \citet{moskalev2022liegg} propose a method for extracting invariance from neural networks, but it cannot handle situations of equivariance. \citet{van2024learning} combine non-equivariant networks with equivariant networks. It assumes that the parameters of the two parts satisfy a given probability distribution, where the hyperparameters of the distribution can reflect the amount of equivariance. Similar to partial G-CNNs \citep{romero2022learning}, it can handle relaxed equivariant tasks well, but still requires specifying group transformations prior.

Some other works are based on dataset to discover symmetries. LieGAN \citep{yang2023generative} uses generative adversarial training to bring the data distributions before and after transformation closer, where the generator is used to generate group actions through group sampling. It eventually learns a set of orthogonal Lie algebra bases, but cannot accurately determine their number. LaLiGAN \citep{yang2023latent} and \citet{tegner2023self} use an encoder-decoder architecture to map the dataset to a latent space with linear symmetries, which can indirectly discover nonlinear symmetries in the original space. \citet{zhou2020meta} present a method for learning and embedding equivariance into networks by learning corresponding parameter sharing patterns from data. SGM \citep{allingham2024generative} captures symmetries under affine and color transformations based on generative models. These methods, based on dataset to discover symmetries, rely on uniform data distribution, while our work can perform well on non-uniform datasets.

\section{Background}

Before describing our method, we will first introduce some preliminary knowledge of group theory.

\subsection{Lie group and Lie algebra}

A Lie group $G$ is a group with a smooth manifold structure. A Lie algebra $\mathfrak{g}$ is a vector space used to describe the local structure near the identity element of a Lie group. For a simply connected Lie group, any group element $g \in G$ can be obtained from a Lie algebra element $A \in \mathfrak{g}$ through the exponential map $\mathrm{exp}: \mathfrak{g} \rightarrow G$, i.e., $g = \mathrm{exp}(A) = \sum_{k=0}^{\infty} \frac{A^k}{k!}$. Decomposing the space where the Lie algebra resides into several Lie algebra bases, we obtain:
\begin{equation}
	\label{eq:group}
	\forall g \in G: \quad g = \mathrm{exp} \left( \sum_{i=1}^D \alpha_i A_i \right),		
\end{equation}
where $D$ is the dimension of the Lie algebra space, $\{A_i\}_{i=1}^D$ are the Lie algebra bases, and $\{\alpha_i\}_{i=1}^D$ denotes their corresponding coefficients.

\subsection{Group representation}

Group representation describes how group elements act on a vector space through linear mappings. Formally, given a simply connected Lie group $G$ and an $n$-dimensional vector space $\mathcal{X}$, a group representation $\rho_{\mathcal{X}}: G \rightarrow GL(n)$ maps group elements to $n$-dimensional invertible matrices. For $g_1, g_2 \in G$ it satisfies $\rho_{\mathcal{X}}(g_1 g_2) = \rho_{\mathcal{X}}(g_1) \rho_{\mathcal{X}}(g_2)$. Similarly, the Lie algebra representation can be defined as $\mathrm{d} \rho_{\mathcal{X}}: \mathfrak{g} \rightarrow \mathfrak{gl}(n)$. We can specialize equation~\eqref{eq:group} to the vector space:
\begin{equation}
	\label{eq:representation}
	\forall g \in G: \quad \rho_{\mathcal{X}}(g) = \mathrm{exp} \left( \sum_{i=1}^D \alpha_i \mathrm{d} \rho_{\mathcal{X}}(A_i) \right).
\end{equation}

\subsection{Equivariance and invariance}

Symmetries can be divided into equivariance and invariance. Given a group $G$, with its group representations on the input space $\mathcal{X} \subseteq \mathbb{R}^n$ and output space $\mathcal{Y} \subseteq \mathbb{R}^m$ denoted as $\rho_{\mathcal{X}}: G \rightarrow GL(n)$ and $\rho_{\mathcal{Y}}: G \rightarrow GL(m)$ respectively. A function $f: \mathcal{X} \rightarrow \mathcal{Y}$ is equivariant if $\forall g \in G, x \in \mathcal{X} : \rho_{\mathcal{Y}}(g) f(x) = f(\rho_{\mathcal{X}}(g) x)$. In particular, when $\rho_{\mathcal{Y}}(g) = I_m$, we call the function $f$ invariant.

\section{Measuring equivariance in neural networks}

We need an efficient method to determine the equivariance of a trained neural network with respect to a given group. We consider the case of a simply connected Lie group, where group elements can be expressed as in equation~\eqref{eq:group}. Naturally, we can perform data augmentation for testing, but this requires sampling of group elements. In equation~\eqref{eq:group}, the Lie algebra bases $\{A_i\}_{i=1}^D$ can determine the information of the group, and different coefficients $\{\alpha_i\}_{i=1}^D$ correspond to different group elements, which means that we need to sample the coefficients. This method requires specifying the distribution of coefficients prior and entails significant computational cost. The following theorem tells us that equivariance in a neural network can be measured through its inputs, outputs, gradients, and Lie algebra bases, without the need for data augmentation.

\begin{restatable}{theorem}{measure}
	\label{thm:measure}
	
	Given a simply connected Lie group $G$, its group representations on the input space $\mathcal{X} \subseteq \mathbb{R}^n$ and output space $\mathcal{Y} \subseteq \mathbb{R}^m$ are denoted respectively as $\rho_{\mathcal{X}}(g): G \rightarrow GL(n)$ and $\rho_{\mathcal{Y}}(g): G \rightarrow GL(m)$. The function $f: \mathcal{X} \rightarrow \mathcal{Y}$ is equivariant:
	\begin{equation}
		\label{eq:measure1}
		\forall g \in G, x \in \mathcal{X}: \quad \rho_{\mathcal{Y}}(g) f(x) = f(\rho_{\mathcal{X}}(g) x),
	\end{equation}
	if and only if the corresponding Lie algebra bases of $\rho_{\mathcal{X}}(g)$ and $\rho_{\mathcal{Y}}(g)$ in equation~\eqref{eq:representation} satisfy:
	\begin{equation}
		\label{eq:measure2}
		\forall i \in \{1, 2, \dots, D\}, x \in \mathcal{X}: \quad \mathrm{d} \rho_{\mathcal{Y}}(A_i) f(x) = \nabla f(x) \mathrm{d} \rho_{\mathcal{X}}(A_i) x.
	\end{equation}
	
\end{restatable}

The complete proof of Theorem~\ref{thm:measure} is provided in \ref{sec:proof-measure}. Note that equation~\eqref{eq:measure2} does not contain coefficients of the Lie algebra bases $\{\alpha_i\}_{i=1}^D$. Hence, to measure equivariance, we solely require knowledge of the group, namely the Lie algebra bases $\{A_i\}_{i=1}^D$. In practice, we directly use training data to obtain the inputs, outputs, and gradients of the neural network.

To discover symmetries, we need to accurately express the function from input to output. However, since the training set is discrete, gradient information cannot be obtained. Therefore, we train a neural network, which can accurately map the input of the training set to the output and is capable of capturing the correct gradient information. Moreover, since we use information from the training set to discover symmetries, we only need it to generalize within the training space $\mathcal{X}$, without requiring it to generalize across the entire space $\mathbb{R}^n$. Compared with methods based on GANs that directly obtain symmetries from the dataset \citep{desai2022symmetry,yang2023generative,yang2023latent}, we use local information of each data point rather than the overall data distribution. Therefore, our method can perform better on non-uniform datasets.

\section{Symmetry discovery for single-channel vector data}
\label{sec:symmetry-discovery}

Our final goal is to discover symmetries via trained neural network, which amounts to solving the space where the Lie algebra resides. A general approach is to optimize the parameters of the Lie algebra bases using gradient descent, where the loss function can be defined as $\mathcal{L}_{equiv} (\{A_i\}_{i=1}^D) = \frac{1}{D} \sum_{i=1}^D ||\mathrm{d} \rho_{\mathcal{Y}}(A_i) f(x) - \nabla f(x) \mathrm{d} \rho_{\mathcal{X}}(A_i) x||_2$ according to Theorem~\ref{thm:measure}. Although we can ensure the orthogonality of Lie algebra bases by adding a regularization term $l_{chreg}(\{A_i\}_{i=1}^D) = \sum_{1 \leq i < j \leq D} R_{ch}(A_i, A_j)$ \citep{yang2023generative}, we cannot exactly determine the dimension $D$ of the Lie algebra space, which may lead to unexpected results.

We propose a mathematical method for solving the Lie algebra bases. Note that equation~\eqref{eq:measure2} is linear with respect to Lie algebra representation. We can transform it into a system of linear equations for solution, and determine the dimension of the Lie algebra space by the number of zero singular values of the coefficient matrix. We summarize the method as the following theorem.

\begin{restatable}{theorem}{single}
	\label{thm:single}
	Suppose that the input space $\mathcal{X} \subseteq \mathbb{R}^n$ and the output space $\mathcal{Y} \subseteq \mathbb{R}^m$ are both single-channel vector spaces. The function $f: \mathcal{X} \rightarrow \mathcal{Y}$ is equivariant with respect to a simply connected Lie group $G$. Sample $N$ data points $\mathcal{D} = \{x^{(i)}\}_{i=1}^N$ from the input space $\mathcal{X}$, and let:
	\begin{equation}
		\label{eq:single}
		C(\mathcal{D}) = 
		\begin{bmatrix}
			-\nabla f(x^{(1)}) \otimes x^{(1) T} & I_m \otimes f^T(x^{(1)}) \\
			-\nabla f(x^{(2)}) \otimes x^{(2) T} & I_m \otimes f^T(x^{(2)}) \\
			\vdots & \vdots \\
			-\nabla f(x^{(N)}) \otimes x^{(N) T} & I_m \otimes f^T(x^{(N)})
		\end{bmatrix}.
	\end{equation}

	Assume that $\mathcal{D}$ makes $\mathrm{rank}(C(\mathcal{D}))$ reach its maximum. Formally speaking, $\forall x \in \mathcal{X}:\mathrm{rank}(C(\mathcal{D} \cup \{x\})) = \mathrm{rank}(C(\mathcal{D}))$. Then the number of zero singular values of $C(\mathcal{D})$ is the dimension of the Lie algebra space of group $G$, and the right singular vectors corresponding to zero singular values are the vector expansion forms of the orthogonal Lie algebra basis representations on $\mathcal{X}$ and $\mathcal{Y}$.
\end{restatable}

To prove Theorem~\ref{thm:single}, we substitute the data samples into equation~\eqref{eq:measure2} and transform it into a system of linear equations:
\begin{equation}
	\label{eq:linear0}
	C(\mathcal{D}) v = 
	\begin{bmatrix}
		-\nabla f(x^{(1)}) \otimes x^{(1) T} & I_m \otimes f^T(x^{(1)}) \\
		-\nabla f(x^{(2)}) \otimes x^{(2) T} & I_m \otimes f^T(x^{(2)}) \\
		\vdots & \vdots \\
		-\nabla f(x^{(N)}) \otimes x^{(N) T} & I_m \otimes f^T(x^{(N)})
	\end{bmatrix}
	\begin{bmatrix}
		\mathrm{vec}(\mathrm{d} \rho_{\mathcal{X}}(A)) \\
		\mathrm{vec}(\mathrm{d} \rho_{\mathcal{Y}}(A))
	\end{bmatrix}
	= 0.
\end{equation}
Then, the problem transforms into finding the bases $\{v_i\}_{i=1}^D$ of the subspace spanned by $v$ that satisfies equation~\eqref{eq:linear0}. The complete proof of Theorem~\ref{thm:single} can be found in \ref{sec:proof-single}. 

In practice, we directly use the training set as the data sample. Due to factors such as data noise and model error, singular values are not exactly zero. So we compare the relative magnitudes of singular values and select those that are much closer to zeros.

\paragraph{Efficiently computing}

In equation~\eqref{eq:single}, we notice that $C(\mathcal{D}) \in \mathbb{R}^{(N \times m) \times (n^2 + m^2)}$. Therefore, when the size $N$ of the training set is large, directly computing the coefficient matrix $C(\mathcal{D})$ will incur significant storage overhead. In practice, we only compute $C(\mathcal{D})^T C(\mathcal{D}) \in \mathbb{R}^{(n^2 + m^2) \times (n^2 + m^2)}$:
\begin{align*}
	& C(\mathcal{D})^T C(\mathcal{D}) = \\ &
	\begin{bmatrix}
		\sum_{i=1}^N \nabla f(x^{(i)})^T \nabla f(x^{(i)}) \otimes x^{(i)} x^{(i) T} & - \sum_{i=1}^N \nabla f(x^{(i)})^T \otimes x^{(i)} f^T(x^{(i)}) \\
		-\sum_{i=1}^N \nabla f(x^{(i)}) \otimes f(x^{(i)}) x^{(i) T} & \sum_{i=1}^N I_m \otimes f(x^{(i)}) f^T(x^{(i)})
	\end{bmatrix}.
\end{align*}
By calculating the eigenvalues and eigenvectors of $C(\mathcal{D})^T C(\mathcal{D})$, we can obtain the singular values and right singular vectors of $C(\mathcal{D})$. Formally, for $C(\mathcal{D}) = U \Sigma V^T$, we have $C(\mathcal{D})^T C(\mathcal{D}) = V \Sigma^2 V^T$.

\section{Extension to different data types}

\subsection{Why the extension is necessary}

We face a crucial question: what kind of symmetries do we want? Recall the motivation behind discovering symmetries, that is to aid the design of equivariant networks, enhance understanding of problems, or contribute to the discovery of scientific laws. Therefore, we need meaningful and interpretable symmetries. Take the two-body problem in LieGAN \citep{yang2023generative} as an example. For the symmetry that enables interactions between the position or momentum of different bodies, we cannot explain its actual physical significance. It is due to the origin of the dataset being at the center of mass and $m_1=m_2$, which results in $q_1=-q_2$ and $p_1=-p_2$. In other words, this symmetry arises from a specific property of the dataset rather than from inherent physical laws. Therefore, for different data types, we impose certain requirements on the form of the Lie algebra basis, which essentially uses our prior knowledge of the structure of the input and output spaces to ensure that symmetries are meaningful.

\subsection{Extension to multi-channel data}

We first discuss the case where both the input and output of the task are multi-channel vectors, which means that independent group transformations are performed within each channel. Formally, the vector space can be decomposed as $V = \bigoplus_{i=1}^c V_i = V_1 \oplus V_2 \oplus \dots \oplus V_c$. Correspondingly, the group representation and Lie algebra representation are $\rho_V(g) = \bigoplus_{i=1}^c \rho_{V_i}(g)$ and $\mathrm{d} \rho_V(A) = \bigoplus_{i=1}^c \mathrm{d} \rho_{V_i}(A)$ respectively, where $\oplus$ is defined as $A \oplus B = 
\begin{bmatrix}
	A & 0 \\
	0 & B
\end{bmatrix}
$ \citep{finzi2021practical}. Similarly, we can obtain the following theorem.

\begin{restatable}{theorem}{multi}
	\label{thm:multi}
	Suppose that the input space $\mathcal{X} = \bigoplus_{i=1}^{c_x} \mathcal{X}_i$ and the output space $\mathcal{Y} = \bigoplus_{i=1}^{c_y} \mathcal{Y}_i$ are both multi-channel vector spaces. The function $f: \mathcal{X} \rightarrow \mathcal{Y}$ is equivariant with respect to a simply connected Lie group $G$. Sample $N$ data points $\mathcal{D} = \{x^{(i)}\}_{i=1}^N$ from the input space $\mathcal{X}$, and let:
	\begin{equation}
		\label{eq:multi}
		C(\mathcal{D}) = 
		\begin{bmatrix}
			C_x^{(1)} & C_y^{(1)} \\
			C_x^{(2)} & C_y^{(2)} \\
			\vdots & \vdots \\
			C_x^{(N)} & C_y^{(N)}
		\end{bmatrix},
	\end{equation}
	where
	\begin{equation*}
		\begin{cases}
			C_x^{(i)} =
			\begin{bmatrix}
				- \nabla f_1(x_1^{(i)}) \otimes x_1^{(i) T} & - \nabla f_1(x_2^{(i)}) \otimes x_2^{(i) T} & \cdots & - \nabla f_1(x_{c_x}^{(i)}) \otimes x_{c_x}^{(i) T} \\
				- \nabla f_2(x_1^{(i)}) \otimes x_1^{(i) T} & - \nabla f_2(x_2^{(i)}) \otimes x_2^{(i) T} & \cdots & - \nabla f_2(x_{c_x}^{(i)}) \otimes x_{c_x}^{(i) T} \\
				\vdots & \vdots & \ddots & \vdots \\
				- \nabla f_{c_y}(x_1^{(i)}) \otimes x_1^{(i) T} & - \nabla f_{c_y}(x_2^{(i)}) \otimes x_2^{(i) T} & \cdots & - \nabla f_{c_y}(x_{c_x}^{(i)}) \otimes x_{c_x}^{(i) T} \\
			\end{bmatrix}, \\
			C_y^{(i)} = \mathrm{diag}[I_{m_1} \otimes f_1^T(x^{(i)}), I_{m_2} \otimes f_2^T(x^{(i)}), \dots, I_{m_{c_y}} \otimes f_{c_y}^T(x^{(i)})],
		\end{cases}
	\end{equation*}
	$x_i$ and $f_i(x)$ are the $i$-th channels of $x$ and $f(x)$ respectively, $\nabla f_i(x_j)$ is the gradient of $f_i(x)$ with respect to $x_j$, and $m_i$ represents the dimension of the $i$-th output channel.
	
	Assume that $\mathcal{D}$ makes $\mathrm{rank}(C(\mathcal{D}))$ reach its maximum. Formally speaking, $\forall x \in \mathcal{X}:\mathrm{rank}(C(\mathcal{D} \cup \{x\})) = \mathrm{rank}(C(\mathcal{D}))$. Then the number of zero singular values of $C(\mathcal{D})$ is the dimension of the Lie algebra space of group $G$, and the right singular vectors corresponding to zero singular values are the vector expansion forms of the orthogonal Lie algebra basis representations on $\{\mathcal{X}_i\}_{i=1}^{c_x}$ and $\{\mathcal{Y}_i\}_{i=1}^{c_y}$.
\end{restatable}

The complete proof of Theorem~\ref{thm:multi} is provided in \ref{sec:proof-multi}. To save memory overhead, we compute $C(\mathcal{D})^T C(\mathcal{D})$ instead of directly storing the coefficient matrix $C(\mathcal{D})$ in equation~\eqref{eq:multi}:
\begin{equation}
	\label{eq:CTC-multi}
	C(\mathcal{D})^T C(\mathcal{D}) = 
	\begin{bmatrix}
		\sum_{i=1}^N C_x^{(i) T} C_x^{(i)} & \sum_{i=1}^N C_x^{(i) T} C_y^{(i)} \\
		\sum_{i=1}^N C_y^{(i) T} C_x^{(i)} & \sum_{i=1}^N C_y^{(i) T} C_y^{(i)}
	\end{bmatrix}.
\end{equation}

\subsection{Extension to tensor data}

Sometimes neural networks take high-dimensional tensors as inputs and outputs rather than vectors. For example, graph neural networks are used to process adjacency matrices representing graph-structured data \citep{kipf2016semi}, and some neural networks are employed for handling the neural weight space \citep{eilertsen2020classifying,unterthiner2020predicting,schurholt2021self}. Many works also attempt to encode symmetries in high-dimensional tensors \citep{satorras2021n,navon2023equivariant}. We now explain how our method for symmetry discovery can be generalized to the tensor case.

For the tensor space $T = \bigotimes_{i=1}^d V_i = V_1 \otimes V_2 \otimes \dots \otimes V_d$, its group representation and Lie algebra representation can be expressed as $\rho_V(g) = \bigotimes_{i=1}^d \rho_{V_i}(g)$ and $\mathrm{d} \rho_V(A) = \overline{\bigoplus}_{i=1}^d \mathrm{d} \rho_{V_i}(A)$, where $\overline{\oplus}$ is defined as $A \overline{\oplus} B = A \otimes I_{b} + I_{a} \otimes B$ \citep{finzi2021practical}. For simplicity, we consider the matrix case, i.e. $\mathcal{X} = \mathcal{X}_1 \otimes \mathcal{X}_2 = \mathbb{R}^{n_1 \times n_2}$ and $\mathcal{Y} = \mathcal{Y}_1 \otimes \mathcal{Y}_2 = \mathbb{R}^{m_1 \times m_2}$. 
The discovery of symmetries in high-dimensional tensors is analogous to matrices. Then we present the following theorem.

\begin{restatable}{theorem}{tensor}
	\label{thm:tensor}
	Suppose that the input space $\mathcal{X} = \mathcal{X}_1 \otimes \mathcal{X}_2 = \mathbb{R}^{n_1 \times n_2}$ and the output space $\mathcal{Y} = \mathcal{Y}_1 \otimes \mathcal{Y}_2 = \mathbb{R}^{m_1 \times m_2}$ are both matrix spaces. The function $F: \mathcal{X} \rightarrow \mathcal{Y}$ is equivariant with respect to a simply connected Lie group $G$. Sample $N$ data points $\mathcal{D} = \{X^{(i)}\}_{i=1}^N$ from the input space $\mathcal{X}$, and let:
	\begin{equation}
		\label{eq:tensor}
		C(\mathcal{D}) = 
		\begin{bmatrix}
			C_{x_1}^{(1)} & C_{x_2}^{(1)} & C_{y_1}^{(1)} & C_{y_2}^{(1)} \\
			C_{x_1}^{(2)} & C_{x_2}^{(2)} & C_{y_1}^{(2)} & C_{y_2}^{(2)} \\
			\vdots & \vdots & \vdots & \vdots \\
			C_{x_1}^{(N)} & C_{x_2}^{(N)} & C_{y_1}^{(N)} & C_{y_2}^{(N)}
		\end{bmatrix},
	\end{equation}
	where
	\begin{equation*}
		\begin{cases}
			C_{x_1}^{(i)} = -\sum_{k=1}^{n_2} \nabla f(X_{\cdot k}^{(i)}) \otimes X_{\cdot k}^{(i) T}, \quad C_{x_2}^{(i)} = -\sum_{k=1}^{n_1} \nabla f(X_{k \cdot}^{(i)}) \otimes X_{k \cdot}^{(i)}, \\
			C_{y_1}^{(i)} = I_{m_1} \otimes F^T(X^{(i)}), \quad C_{y_2}^{(i)} = 
			\begin{bmatrix}
				I_{m_2} \otimes F_{1 \cdot}(X^{(i)}) \\
				I_{m_2} \otimes F_{2 \cdot}(X^{(i)}) \\
				\vdots \\
				I_{m_2} \otimes F_{m_1 \cdot}(X^{(i)}) \\
			\end{bmatrix},
		\end{cases}
	\end{equation*}
	$f(X) \in \mathbb{R}^{(m_1 \times m_2) \times 1}$ is the vectorized form of the output matrix $F(X) \in \mathcal{Y}$, $X_{k \cdot} \in \mathbb{R}^{1 \times n_2}$ and $F_{k \cdot}(X) \in \mathbb{R}^{1 \times m_2}$ are the $k$-th row vectors of $X$ and $F(X)$ respectively, $X_{\cdot k} \in \mathbb{R}^{n_1 \times 1}$ is the $k$-th column vector of $X$, $\nabla f(X_{\cdot k}) \in \mathbb{R}^{(m_1 \times m_2) \times n_1}$ and $\nabla f(X_{k \cdot}) \in \mathbb{R}^{(m_1 \times m_2) \times n_2}$ denote the gradient of $f(X)$ with respect to $X_{\cdot k}$ and $X_{k \cdot}$ respectively.
	
	Assume that $\mathcal{D}$ makes $\mathrm{rank}(C(\mathcal{D}))$ reach its maximum. Formally speaking, $\forall x \in \mathcal{X}:\mathrm{rank}(C(\mathcal{D} \cup \{x\})) = \mathrm{rank}(C(\mathcal{D}))$. Then the number of zero singular values of $C(\mathcal{D})$ is the dimension of the Lie algebra space of group $G$, and the right singular vectors corresponding to zero singular values are the vector expansion forms of the orthogonal Lie algebra basis representations on $\{\mathcal{X}_i\}_{i=1}^2$ and $\{\mathcal{Y}_i\}_{i=1}^2$.
\end{restatable}

We provide the complete proof of Theorem~\ref{thm:tensor} in \ref{sec:proof-tensor}. The expression $C(\mathcal{D})^T C(\mathcal{D})$ can be computed in the same way as equation~\eqref{eq:CTC-multi}.

\section{Experiments}
\label{sec:experiments}

We evaluate \model~on tasks with symmetries. We will validate that \model~can perform well on non-uniform datasets, handle multi-channel and tensor cases, and accurately determine the number of Lie algebra bases. As mentioned in Section~\ref{sec:related-work}, many previous works attempted to discover symmetries \citep{benton2020learning,zhou2020meta,desai2022symmetry,moskalev2022liegg}, but LieGAN \citep{yang2023generative} has already proven to achieve state-of-the-art results. Therefore, we only compare with LieGAN.

\subsection{Quantitative criteria}
\label{sec:quantitative}

We hope to quantitatively analyze the accuracy of Lie algebra bases. The ideal Lie algebra bases span the same space as the real Lie algebra space and they are pairwise orthogonal. Therefore, we will define space error and orthogonality error to evaluate the symmetry discovery method. Suppose $\{v_i\}_{i=1}^{D}$ are the normalized vector expansion forms of the discovered Lie algebra bases, with their ground truth being $\{v_i^*\}_{i=1}^{D^*}$.

\subsubsection{Space error}
Two sets of bases $\{v_i\}_{i=1}^D$ and $\{v_i^*\}_{i=1}^{D^*}$ spanning the same space in $\mathbb{R}^n$ means that each $v_i$ can be linearly represented by $\{v_i^*\}_{i=1}^{D^*}$, and conversely, each $v_i^*$ can also be linearly represented by $\{v_i\}_{i=1}^D$. Formally:
\begin{equation}
	\label{eq:over-determined}
	\begin{cases}
		\exists X \in \mathbb{R}^{D \times D^*}: \quad V^* = V X, \\
		\exists Y \in \mathbb{R}^{D^* \times D}: \quad V = V^* Y,
	\end{cases}
\end{equation}
where $V \in \mathbb{R}^{n \times D}$ and $V^* \in \mathbb{R}^{n \times D^*}$ are matrices with $\{v_i\}_{i=1}^D$ and $\{v_i^*\}_{i=1}^{D^*}$ as column vectors, respectively.

Therefore, we can define the space error as the residual of the overdetermined system of linear equations~\eqref{eq:over-determined}:
\begin{equation*}
	E_{space} = \min_X ||VX - V^*||_F^2 + \min_Y ||V^* Y - V||_F^2,
\end{equation*}
where the optimal solution can be obtained by solving the linear equation systems $V^T V^* = V^T V X$ and $V^{*T} V = V^{*T} V^* Y$.

\subsubsection{Orthogonality error}

Evaluating the accuracy of basis vectors solely using space error is insufficient. For instance, 9 basis vectors can also span the correct 7-dimensional subspace, while two of them are redundant. Therefore, we need to penalize such cases with orthogonal error. We define orthogonal error as the sum of the absolute values of the inner products of the basis vectors pairwise:
\begin{equation*}
	E_{orth} = \sum_{i=1}^{D-1} \sum_{j=i+1}^D |\langle v_i, v_j \rangle|.
\end{equation*}

\subsection{Two-body problem}
\label{sec:experiments-2body}

We first consider the two-body problem \citep{greydanus2019hamiltonian}. It studies the motion of two particles on a plane under the influence of gravity, with their center of mass as the origin. The task takes the motion states $q_1, p_1, q_2, p_2$ of the two particles at time $t$ as inputs and predicts the motion states of the two particles at time $t+1$, where $q_i \in \mathbb{R}^2$ and $p_i \in \mathbb{R}^2$ represent the position and momentum coordinates of the $i$-th particle, respectively.

The two-body problem has $\mathrm{SO(2)}$ equivariance, meaning that rotating the input at time $t$ by a certain angle will correspondingly rotate the prediction at time $t+1$ by the same angle. The four channels $q_1, p_1, q_2, p_2$ undergo group transformations independently. Formally, the number of Lie algebra basis is $D=1$, and $\mathrm{d} \rho_{\mathcal{X}}(A) = \mathrm{d} \rho_{\mathcal{Y}}(A) = \bigoplus_{i=1}^4 
\begin{bmatrix}
	0 & -1 \\
	1 & 0
\end{bmatrix}
$.

We will validate \model~on the two-body problem to demonstrate its ability to handle multi-channel cases. Additionally, we will shuffle the data distribution to evaluate its performance on the non-uniform dataset. For comparison, we specify the Lie algebra basis number of LieGAN as 1. We provide implementation details in \ref{sec:implementation-2body}. We also present the space error and orthogonality error in Table~\ref{tab:2body}, and the visualization results in Figure~\ref{fig:2body}.

\begin{table}[ht]
	\centering
	\resizebox{\linewidth}{!}{
	\begin{tabular}{c|c|ll|ll}
		\toprule
		Dataset & Method & \multicolumn{2}{c|}{Space error} & \multicolumn{2}{|c}{Orthogonality error} \\
		\midrule
		\multirow{2}{*}{Uniform} & \model & $2.65 \times 10^{-2}$ & \textcolor{gray}{$(4.35 \times 10^{-4} \sim 1.04 \times 10^{-1})$} & $0$ & \textcolor{gray}{$(0 \sim 0)$} \\
		& LieGAN & $3.33 \times 10^{-4}$ & \textcolor{gray}{$(1.97 \times 10^{-4} \sim 5.93 \times 10^{-4})$} & $0$ & \textcolor{gray}{$(0 \sim 0)$} \\
		\midrule
		\multirow{2}{*}{Non-uniform} & \model & $7.17 \times 10^{-2}$ & \textcolor{gray}{$(2.50 \times 10^{-2} \sim 1.85 \times 10^{-1})$} & $0$ & \textcolor{gray}{$(0 \sim 0)$} \\
		& LieGAN & $9.35 \times 10^{-1}$ & \textcolor{gray}{$(5.97 \times 10^{-1} \sim 1.35)$} & $0$ & \textcolor{gray}{$(0 \sim 0)$} \\
		\bottomrule
	\end{tabular}}
	\caption{The space error and orthogonality error of symmetry discovery in the two-body problem. We present the results in the format of averages \textcolor{gray}{(minimum $\sim$ maximum)}. \model~exhibits strong robustness on the non-uniform dataset.}
	\label{tab:2body}
\end{table}

\begin{figure}[ht]
	\centering
	\begin{subfigure}[b]{.3\textwidth}
		\includegraphics[width=\textwidth]{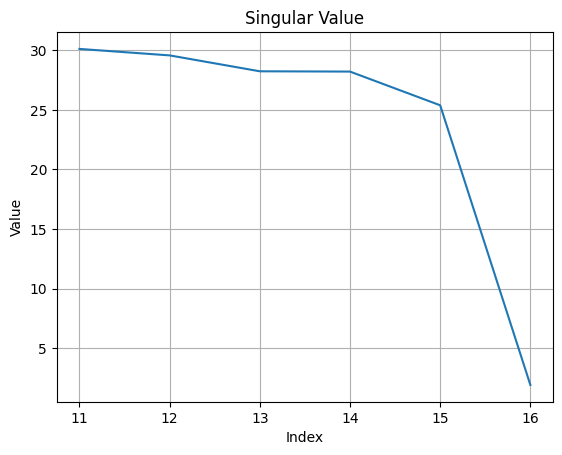}
		\caption{Uniform, singular value}
	\end{subfigure}
	\hfill
	\begin{subfigure}[b]{.3\textwidth}
		\includegraphics[width=\textwidth]{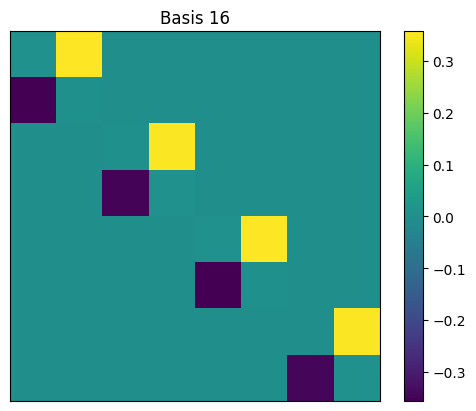}
		\caption{Uniform, \model}
	\end{subfigure}
	\hfill
	\begin{subfigure}[b]{.3\textwidth}
		\includegraphics[width=\textwidth]{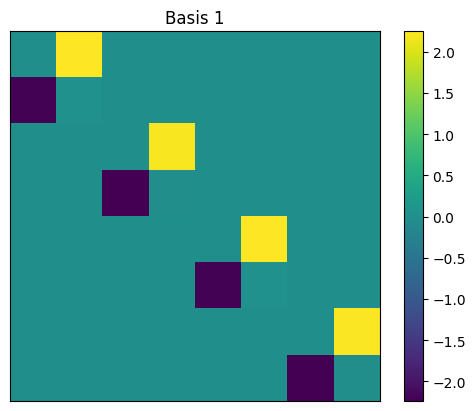}
		\caption{Uniform, LieGAN}
	\end{subfigure}
	\\
	\begin{subfigure}[b]{.3\textwidth}
		\includegraphics[width=\textwidth]{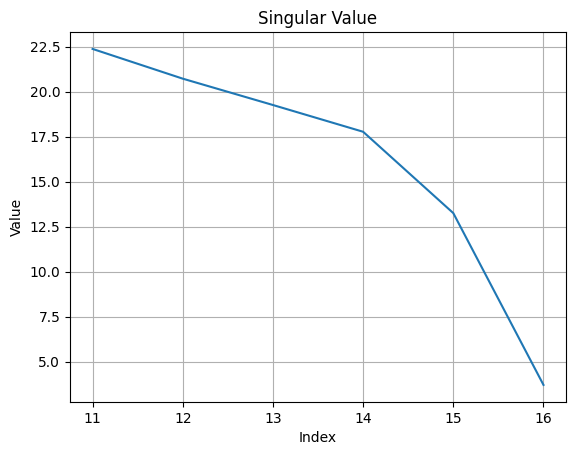}
		\caption{Non-uniform, singular value}
	\end{subfigure}
	\hfill
	\begin{subfigure}[b]{.3\textwidth}
		\includegraphics[width=\textwidth]{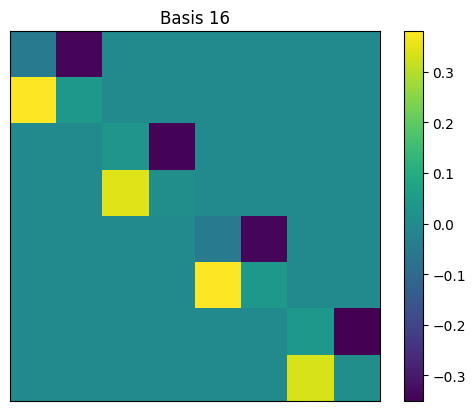}
		\caption{Non-uniform, \model}
	\end{subfigure}
	\hfill
	\begin{subfigure}[b]{.3\textwidth}
		\includegraphics[width=\textwidth]{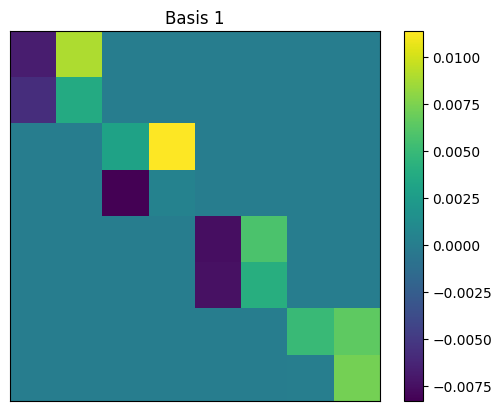}
		\caption{Non-uniform, LieGAN}
	\end{subfigure}
	\caption{The visualization results of symmetry discovery in the two-body problem. (a): The top 6 smallest singular values solved by \model~on the uniform dataset, which are arranged in descending order. (b): The Lie algebra basis corresponding to the minimum singular value solved by \model~on the uniform dataset. (c): The Lie algebra basis learned by LieGAN with 1 channel on the uniform dataset. (d-f): Results corresponding to (a-c) on the non-uniform dataset.}
	\label{fig:2body}
\end{figure}

As shown in Table~\ref{tab:2body} and Figure~\ref{fig:2body}, both \model~and LieGAN can discover the correct Lie algebra basis on the uniform dataset. Note that when the Lie algebra basis differs by a coefficient, the space it spans remains the same. There is only one Lie algebra basis, so the orthogonality error is always zero. Furthermore, \model~identifies one singular value that is almost zero, indicating that the number of Lie algebra basis is $D=1$, whereas LieGAN requires manual specification. When we disrupt the data distribution, LieGAN fails, but \model~still performs well, demonstrating stronger robustness on the non-uniform dataset.

\subsection{The moment of inertia matrix prediction}
\label{sec:experiments-inertia}

We will now explore symmetries in predicting the moment of inertia matrix \citep{finzi2021practical}. Given the masses $m_i \in \mathbb{R}$ and position coordinates $x_i \in \mathbb{R}^3$ of several particles on a rigid body, the moment of inertia matrix is defined as $M(\{m_i, x_i\}_i) = \sum_i m_i (x_i^T x_i I - x_i x_i^T)$. This task exhibits $\mathrm{SO(3)}$ and scaling equivariance. For $R \in \mathrm{SO(3)}$ and $S \in \{s I | s \in \mathbb{R}\}$, $M(\{m_i, R x_i\}_i) = R M(\{m_i, x_i\}_i) R^T$ and $M(\{m_i, S x_i\}_i) = S M(\{m_i, x_i\}_i) S^T$ hold. Formally, the input is a multi-channel vector $\mathcal{X} = \bigoplus_i \mathcal{X}_i$, and the output is a matrix $\mathcal{Y} = \mathcal{Y}_1 \otimes \mathcal{Y}_2$, where $\mathcal{X}_i = \mathbb{R}^3$ and $\mathcal{Y}_i = \mathbb{R}^3$. 

We use the moment of inertia matrix prediction to verify that \model~can handle tensor cases. We also perform an ablation study to demonstrate that processing the output as a matrix rather than an unfolded vector will be better. The implementation details can be found in \ref{sec:implementation-inertia}. We also present the space error and orthogonality error in Table~\ref{tab:inertia}, and the visualization results in Figures \ref{fig:inertia} and \ref{fig:ablation}.

\begin{table}[ht]
	\centering
	\resizebox{\linewidth}{!}{
	\begin{tabular}{c|ll|ll}
		\toprule
		Method & \multicolumn{2}{c|}{Space error} & \multicolumn{2}{|c}{Orthogonality error} \\
		\midrule
		\model, tensor & $7.08 \times 10^{-5}$ & \textcolor{gray}{$(6.66 \times 10^{-5} \sim 7.55 \times 10^{-5})$} & $1.17 \times 10^{-1}$ & \textcolor{gray}{$(5.91 \times 10^{-2} \sim 2.22 \times 10^{-1})$} \\
		\model, vector & $3.44 \times 10^1$ & \textcolor{gray}{$(3.43 \times 10^1 \sim 3.45 \times 10^1)$} & $4.56 \times 10^{-5}$ & \textcolor{gray}{$(1.78 \times 10^{-5} \sim 8.88 \times 10^{-5})$} \\
		\bottomrule
	\end{tabular}}
	\caption{The space error and orthogonality error of symmetry discovery in predicting the moment of inertia matrix. We present the results in the format of averages \textcolor{gray}{(minimum $\sim$ maximum)}. Using the tensor form of \model~allows for accurate computation of the Lie algebra space, while using the vector form of \model~fails.}
	\label{tab:inertia}
\end{table}

\begin{figure}[ht]
	\centering
	\begin{subfigure}[b]{.45\textwidth}
		\includegraphics[width=\textwidth]{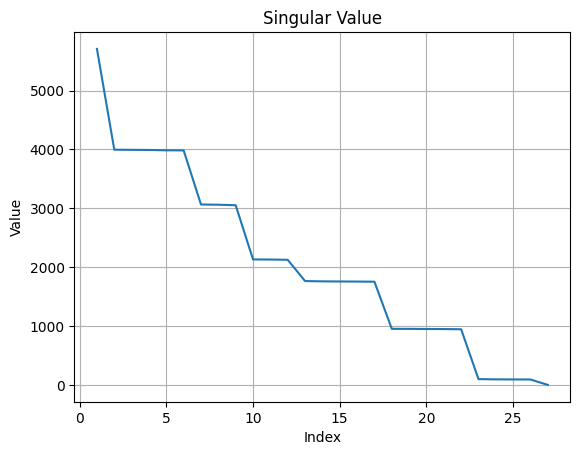}
		\caption{Singular value}
	\end{subfigure}
	\hfill
	\begin{subfigure}[b]{.45\textwidth}
		\includegraphics[width=\textwidth]{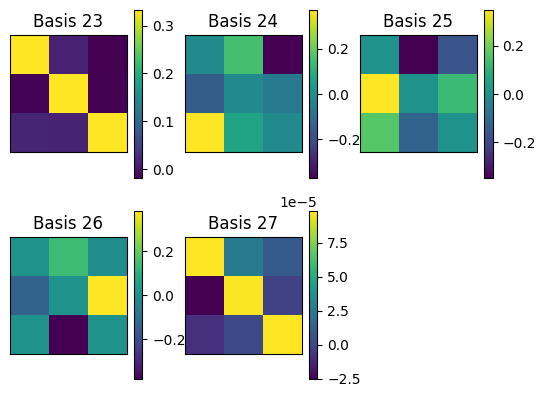}
		\caption{Lie algebra basis representation in $\mathcal{X}_i$}
	\end{subfigure}
	\\
	\begin{subfigure}[b]{.45\textwidth}
		\includegraphics[width=\textwidth]{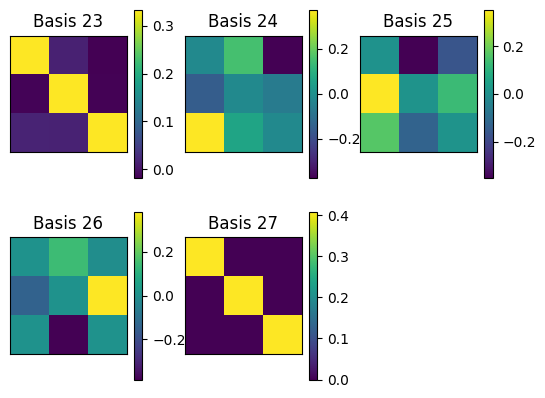}
		\caption{Lie algebra basis representation in $\mathcal{Y}_1$}
	\end{subfigure}
	\hfill
	\begin{subfigure}[b]{.45\textwidth}
		\includegraphics[width=\textwidth]{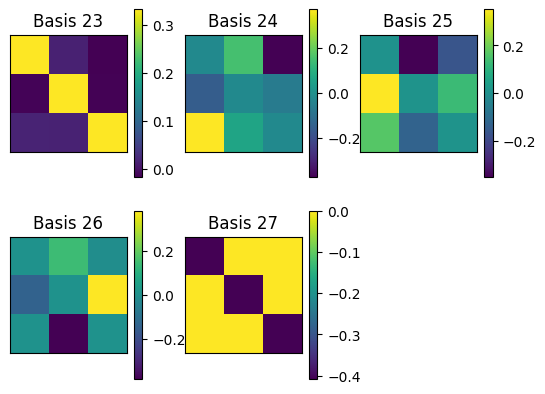}
		\caption{Lie algebra basis representation in $\mathcal{Y}_2$}
	\end{subfigure}
	\caption{The visualization results of symmetry discovery in predicting the moment of inertia matrix of \model. (a): The computed singular values, which are arranged in descending order. (b-d): Lie algebra bases corresponding to five nearly zero singular values in spaces $\mathcal{X}_i, \mathcal{Y}_1, \mathcal{Y}_2$, where basis $i$ corresponds to the singular value with index $i$.}
	\label{fig:inertia}
\end{figure}

\begin{figure}[ht]
	\centering
	\begin{subfigure}[b]{.45\textwidth}
		\includegraphics[width=\textwidth]{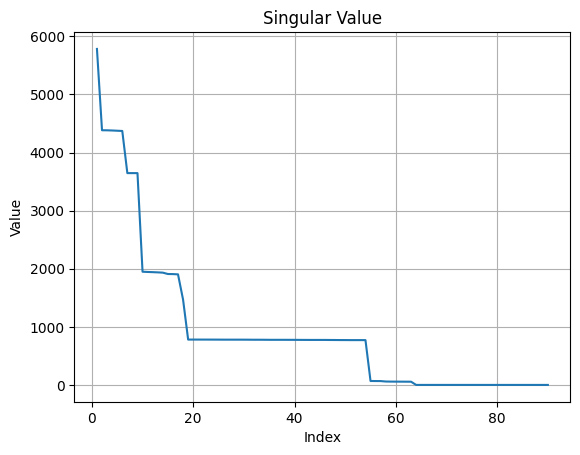}
		\caption{Singular value}
	\end{subfigure}
	\hfill
	\begin{subfigure}[b]{.45\textwidth}
		\includegraphics[width=\textwidth]{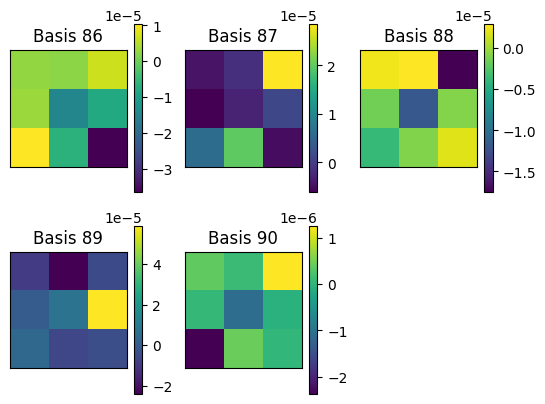}
		\caption{Lie algebra basis representation in $\mathcal{X}_i$}
	\end{subfigure}
	\\
	\begin{subfigure}[b]{.45\textwidth}
		\includegraphics[width=\textwidth]{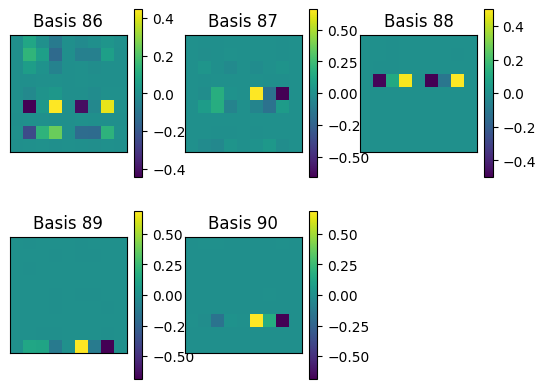}
		\caption{Lie algebra basis representation in $\mathcal{Y}$}
	\end{subfigure}
	\caption{The visualization results of the ablation study on symmetry discovery in predicting the moment of inertia matrix. (a): The computed singular values, which are arranged in descending order. (b-c): Lie algebra bases corresponding to the five smallest singular values in spaces $\mathcal{X}_i$ and $\mathcal{Y}$, where basis $i$ corresponds to the singular value with index $i$.}
	\label{fig:ablation}
\end{figure}

As shown in Figure~\ref{fig:inertia}, \model~finds 5 singular values that are almost zeros. Bases 24-26 correspond to $\mathrm{SO(3)}$ equivariance, and basis 23 corresponds to scaling equivariance. For basis 27, $\mathrm{d} \rho_{\mathcal{Y}}(A_1) = \mathrm{d} \rho_{\mathcal{Y}1}(A_1) \overline{\oplus} \mathrm{d} \rho_{\mathcal{Y}_2}(A_1) = \mathrm{diag}(a, a, a) \overline{\oplus} \mathrm{diag}(-a, -a, -a) = \mathbf{0}$ and $\mathrm{d} \rho_{\mathcal{X}}(A_1) = \mathbf{0}$. Therefore, basis 27 corresponds to the case where group representations of the input and output spaces are both trivial.

On the other hand, when we treat the moment of inertia matrix as a flattened vector, \model~computes 26 nearly zero singular values and obtains a series of meaningless Lie algebra bases as shown in Figure~\ref{fig:ablation}. In Table~\ref{tab:inertia}, we can confirm that using the tensor form of \model~allows us to solve the correct Lie algebra space, while using the vector form of \model~leads to completely incorrect Lie algebra space. Therefore, when the input or output is in tensor form, employing Theorem~\ref{thm:tensor} instead of flattening into vectors and then using Theorem~\ref{thm:single} to discover symmetries (essentially imposing formal constraints on Lie algebra representations) can help us more accurately identify the correct Lie algebra space. The reason why the orthogonal error in the tensor case is greater than in the vector case is that the orthogonality of $\mathrm{d} \rho_{\mathcal{Y}_1}(A)$ and $\mathrm{d} \rho_{\mathcal{Y}_2}(A)$ does not guarantee the orthogonality of $\mathrm{d} \rho_{\mathcal{Y}}(A) = \mathrm{d} \rho_{\mathcal{Y}_1}(A) \overline{\oplus} \mathrm{d} \rho_{\mathcal{Y}_2}(A)$. 

\subsection{Top quark tagging}
\label{sec:experiments-top-quark-tagging}

Top quark tagging \citep{yang2023generative} is a task of classifying hadronic tops from QCD background. It takes as input the four-momenta $p^{\mu}_i = (p^0_i, p^1_i, p^2_i, p^3_i) \in \mathbb{R}^4$ of the 20 jet constituents with the highest transverse momentum $p_T$ produced in particle collisions and predicts labels for the particles, where $p^0_i \in \mathbb{R}$ is the energy component of the $i$-th jet constituent, and $(p^1_i, p^2_i, p^3_i) \in \mathbb{R}^3$ are the spatial momentum components. This task possesses $\mathrm{SO(1,3)^+}$ and scaling invariance. Rotations in three-dimensional space, boosts, and scaling transformations will not alter the classification result.

We use top quark tagging to illustrate the advantage of \model~in accurately determining the number of Lie algebra bases, while LieGAN requires manual specification. We will set different numbers of Lie algebra bases for LieGAN and compare the differences in their results. We provide implementation details in \ref{sec:implementation-top-quark-tagging}. We also present the space error and orthogonality error in Table~\ref{tab:top-quark-tagging}, and the visualization results in Figure~\ref{fig:top-quark-tagging}.

\begin{table}[ht]
	\centering
	\resizebox{\linewidth}{!}{
	\begin{tabular}{c|ll|ll}
		\toprule
		Method & \multicolumn{2}{c|}{Space error} & \multicolumn{2}{|c}{Orthogonality error} \\
		\midrule
		\model & $1.47 \times 10^{-3}$ & \textcolor{gray}{$(1.19 \times 10^{-3} \sim 2.11 \times 10^{-3})$} & $1.72 \times 10^{-6}$ & \textcolor{gray}{$(1.46 \times 10^{-6} \sim 2.07 \times 10^{-6})$} \\
		LieGAN, 7 channels & $2.49$ & \textcolor{gray}{$(1.28 \sim 4.07)$} & $5.82 \times 10^{-1}$ & \textcolor{gray}{$(2.14 \times 10^{-1} \sim 1.26)$} \\
		LieGAN, 9 channels & $2.22$ & \textcolor{gray}{$(1.65 \sim 2.75)$} & $1.13$ & \textcolor{gray}{$(5.54 \times 10^{-1} \sim 1.43)$} \\
		\bottomrule
	\end{tabular}}
	\caption{The space error and orthogonality error of symmetry discovery in top quark tagging. We present the results in the format of averages \textcolor{gray}{(minimum $\sim$ maximum)}. Compared with LieGAN, the Lie algebra space solved by \model~is more accurate, and the Lie algebra bases are almost strictly orthogonal. What's more, incorrectly specifying the number of Lie algebra bases for LieGAN will make the results even worse.}
	\label{tab:top-quark-tagging}
\end{table}

\begin{figure}[ht]
	\centering
	\begin{subfigure}[b]{.45\textwidth}
		\includegraphics[width=\textwidth]{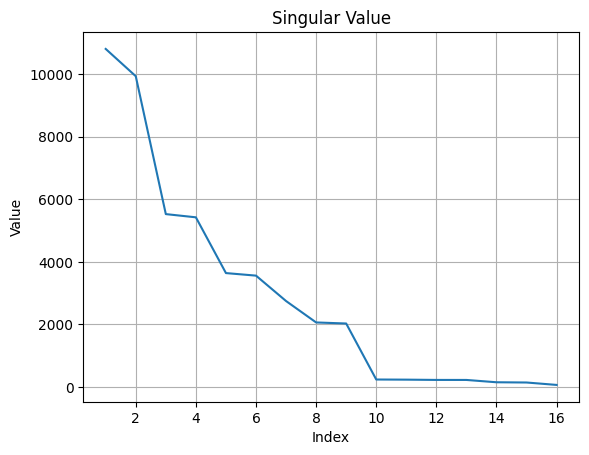}
		\caption{\model, singular value}
		\label{fig:top-quark-tagging-a}
	\end{subfigure}
	\hfill
	\begin{subfigure}[b]{.45\textwidth}
		\includegraphics[width=\textwidth]{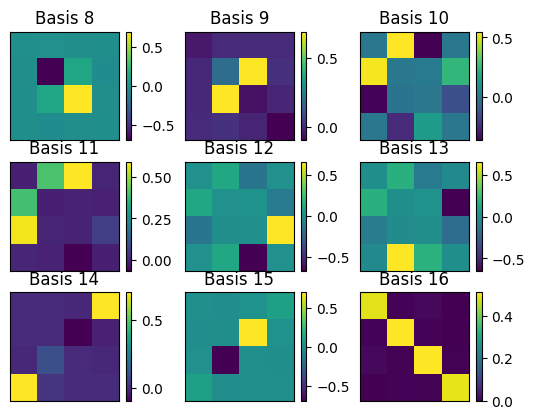}
		\caption{\model, Lie algebra bases}
		\label{fig:top-quark-tagging-b}
	\end{subfigure}
	\\
	\begin{subfigure}[b]{.45\textwidth}
		\includegraphics[width=\textwidth]{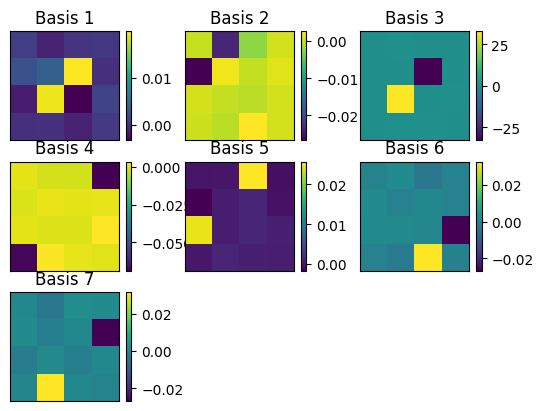}
		\caption{LieGAN, 7 channels}
		\label{fig:top-quark-tagging-c}
	\end{subfigure}
	\hfill
	\begin{subfigure}[b]{.45\textwidth}
		\includegraphics[width=\textwidth]{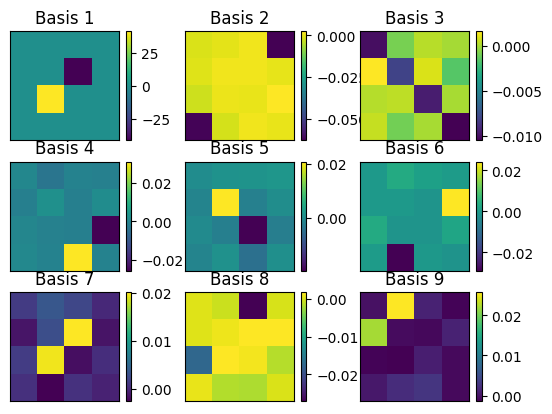}
		\caption{LieGAN, 9 channels}
		\label{fig:top-quark-tagging-d}
	\end{subfigure}
	\caption{The visualization results of symmetry discovery in top quark tagging. (a): The singular values obtained by \model, which are sorted in descending order. (b): Lie algebra bases solved by \model, where basis $i$ corresponds to the $i$-th largest singular value. (c): The Lie algebra bases learned by LieGAN with 7 channels. (d): The Lie algebra bases learned by LieGAN with 9 channels.}
	\label{fig:top-quark-tagging}
\end{figure}

As shown in Figure~\ref{fig:top-quark-tagging}, \model~solves for 7 singular values that are almost zeros, indicating that the number of Lie algebra bases is 7. Therefore, bases 10-16 in Figure~\ref{fig:top-quark-tagging-b} are correct while bases 8-9 are not. Among them, basis 16 corresponds to scaling, bases 12, 13 and 15 correspond to rotations in three-dimensional space, and bases 10, 11 and 14 correspond to boosts. On the other hand, although we accurately set the number of Lie algebra bases for LieGAN to 7, \model~can still outperform it in terms of accuracy and orthogonality as shown in Table~\ref{tab:top-quark-tagging}. Furthermore, when we mistakenly set its number of Lie algebra bases to 9, LieGAN results in worse performance. It learns incorrect bases 5 and 7 as shown in Figure~\ref{fig:top-quark-tagging-d}, and we cannot distinguish them. In short, the Lie algebra bases obtained by \model~are ordered, and the singular value can tell us which ones are valid, while LieGAN cannot achieve this because the Lie algebra bases it obtains are disordered.

\section{Conclusion}

This work proposes a mathematical approach to discover symmetries via trained neural networks. The theorem we propose allows for measuring the equivariance in neural networks without the need for data augmentation or group sampling. Based on this theorem, we obtain the Lie algebra space by solving a system of linear equations, where the number of zero singular values corresponds to the dimension of the Lie algebra space. Furthermore, we extend the method to the multi-channel and tensor cases to adapt to different data types. Compared with methods based on GANs that discover symmetries from the dataset, our approach does not rely on data distribution, thus performing well on non-uniform datasets.

\appendix
\section{Complete proofs}
\label{sec:proofs}

\subsection{Proof of Theorem~\ref{thm:measure}}
\label{sec:proof-measure}

\measure*

\begin{proof}
	First, we prove equation~\eqref{eq:measure1} $\Rightarrow$ equation~\eqref{eq:measure2}. Substitute equation~\eqref{eq:representation} into equation~\eqref{eq:measure1}, and note that different coefficients of the Lie algebra bases correspond to different group elements:
	\begin{align}
		\label{eq:proof1}
		& \forall \alpha \in \mathbb{R}^D, x \in \mathcal{X}: \nonumber \\
		& h(\alpha) = \mathrm{exp} \left( \sum_{i=1}^D \alpha_i \mathrm{d} \rho_{\mathcal{Y}}(A_i) \right) f(x) - f \left( \mathrm{exp} \left(\sum_{i=1}^D \alpha_i \mathrm{d} \rho_{\mathcal{X}}(A_i) \right) x \right) = 0.
	\end{align}
	Taking the derivative element-wise with respect to $\alpha$, and then setting $\alpha=0$, we obtain:
	\begin{equation*}
		\forall i \in \{1, 2, \dots, D\}, x \in \mathcal{X}: \quad \frac{\partial h}{\partial \alpha_i} \bigg\vert_{\alpha=0} = \mathrm{d} \rho_{\mathcal{Y}}(A_i) f(x) - \nabla f(x) \mathrm{d} \rho_{\mathcal{X}}(A_i) x = 0.
	\end{equation*}

	Then, we prove equation~\eqref{eq:measure2} $\Rightarrow$ equation~\eqref{eq:measure1}. For convenience, we denote $B = \sum_{i=1}^D \alpha_i A_i$. By the linearity of the Lie algebra representation, we have $\mathrm{d} \rho_{\mathcal{X}}(B) = \sum_{i=1}^D \alpha_i \mathrm{d} \rho_{\mathcal{X}}(A_i)$ and $\mathrm{d} \rho_{\mathcal{Y}}(B) = \sum_{i=1}^D \alpha_i \mathrm{d} \rho_{\mathcal{Y}}(A_i)$. What's more:
	\begin{equation*}
		\mathrm{d} \rho_{\mathcal{Y}}(B) f(x) = \nabla f(x) \mathrm{d} \rho_{\mathcal{X}}(B) x.
	\end{equation*}
	Replace $x$ with $\mathrm{exp}(\beta \cdot \mathrm{d} \rho_{\mathcal{X}}(B)) x$:
	\begin{equation}
		\label{eq:proof2}
		\mathrm{d} \rho_{\mathcal{Y}}(B) f(\mathrm{exp}(\beta \cdot \mathrm{d} \rho_{\mathcal{X}}(B)) x) = \nabla f(\mathrm{exp}(\beta \cdot \mathrm{d} \rho_{\mathcal{X}}(B)) x) \mathrm{d} \rho_{\mathcal{X}}(B) \mathrm{exp}(\beta \cdot \mathrm{d} \rho_{\mathcal{X}}(B)) x.
	\end{equation}
	Define $\tilde{h}: \mathbb{R} \rightarrow \mathbb{R}^m$ as:
	\begin{equation*}
		\tilde{h}(\beta) = \mathrm{exp}(\beta \cdot \mathrm{d} \rho_{\mathcal{Y}}(B)) f(x) - f(\mathrm{exp}(\beta \cdot \mathrm{d} \rho_{\mathcal{X}}(B)) x).
	\end{equation*}
	Taking the derivative with respect to $\beta$:
	\begin{align*}
		\frac{\mathrm{d} \tilde{h}}{\mathrm{d} \beta} &= \mathrm{d} \rho_{\mathcal{Y}}(B) \mathrm{exp}(\beta \cdot \mathrm{d} \rho_{\mathcal{Y}}(B)) f(x) \\
		& - \nabla f(\mathrm{exp}(\beta \cdot \mathrm{d} \rho_{\mathcal{X}}(B)) x) \mathrm{d} \rho_{\mathcal{X}}(B) \mathrm{exp}(\beta \cdot \mathrm{d} \rho_{\mathcal{X}}(B)) x.
	\end{align*}
	Substitute equation~\eqref{eq:proof2}:
	\begin{align*}
		\frac{\mathrm{d} \tilde{h}}{\mathrm{d} \beta} &= \mathrm{d} \rho_{\mathcal{Y}}(B) \mathrm{exp}(\beta \cdot \mathrm{d} \rho_{\mathcal{Y}}(B)) f(x) - \mathrm{d} \rho_{\mathcal{Y}}(B) f(\mathrm{exp}(\beta \cdot \mathrm{d} \rho_{\mathcal{X}}(B)) x) \\
		&= \mathrm{d} \rho_{\mathcal{Y}}(B) \tilde{h}(\beta).
	\end{align*}
	Solving the differential equation, and note that $\tilde{h}(0)=0$:
	\begin{equation*}
		\tilde{h}(\beta) = \mathrm{exp}(\beta \cdot \mathrm{d} \rho_{\mathcal{Y}}(B)) \tilde{h}(0) = 0.
	\end{equation*}
	Let $\beta=1$, we obtain:
	\begin{equation*}
		\mathrm{exp}(\mathrm{d} \rho_{\mathcal{Y}}(B)) f(x) - f(\mathrm{exp}(\mathrm{d} \rho_{\mathcal{X}}(B)) x) = 0,
	\end{equation*}
	which is equivalent to equation~\eqref{eq:proof1}. Therefore, equation~\eqref{eq:measure1} is proven.
\end{proof}

\subsection{Proof of Theorem~\ref{thm:single}}
\label{sec:proof-single}

\single*

\begin{proof}
	Flatten the Lie algebra representation in equation~\eqref{eq:measure2} from matrices into vectors:
	\begin{equation*}
		\forall x \in \mathcal{X}: \quad
		[I_m \otimes f^T(x)] \mathrm{vec}(\mathrm{d} \rho_{\mathcal{Y}}(A)) = [\nabla f(x) \otimes x^T] \mathrm{vec}(\mathrm{d} \rho_{\mathcal{X}}(A)).
	\end{equation*}
	Note that equivariance implies a one-to-one correspondence between the Lie algebra representations of the input space and the output space. We concatenate them together:
	\begin{equation}
		\label{eq:concat1}
		\forall x \in \mathcal{X}: \quad
		\begin{bmatrix}
			-\nabla f(x) \otimes x^T & I_m \otimes f^T(x)
		\end{bmatrix}
		\begin{bmatrix}
			\mathrm{vec}(\mathrm{d} \rho_{\mathcal{X}}(A)) \\
			\mathrm{vec}(\mathrm{d} \rho_{\mathcal{Y}}(A))
		\end{bmatrix}
		=0.
	\end{equation}
	Sample $N$ data points $\{x^{(i)}\}_{i=1}^N$ from the input space $\mathcal{X}$, then equation~\eqref{eq:concat1} holds for any element $x^{(i)}$. We obtain a system of linear equations:
	\begin{equation}
		\label{eq:linear1}
		C(\mathcal{D}) v = 
		\begin{bmatrix}
			-\nabla f(x^{(1)}) \otimes x^{(1) T} & I_m \otimes f^T(x^{(1)}) \\
			-\nabla f(x^{(2)}) \otimes x^{(2) T} & I_m \otimes f^T(x^{(2)}) \\
			\vdots & \vdots \\
			-\nabla f(x^{(N)}) \otimes x^{(N) T} & I_m \otimes f^T(x^{(N)})
		\end{bmatrix}
		\begin{bmatrix}
			\mathrm{vec}(\mathrm{d} \rho_{\mathcal{X}}(A)) \\
			\mathrm{vec}(\mathrm{d} \rho_{\mathcal{Y}}(A))
		\end{bmatrix}
		= 0.
	\end{equation}
	Obviously, $\langle \mathrm{d} \rho(A_i), \mathrm{d} \rho(A_j) \rangle = 0$ is equivalent to $\langle v_i, v_j \rangle = 0$. Therefore, the problem transforms into finding the bases $\{v_i\}_{i=1}^D$ of the subspace spanned by $v$ that satisfies equation~\eqref{eq:linear1}.
	
	Perform singular value decomposition on the coefficient matrix $C(\mathcal{D})$:
	\begin{equation*}
		C(\mathcal{D}) v = U
		\begin{bmatrix}
			\Sigma & 0 \\
			0 & 0
		\end{bmatrix}
		\begin{bmatrix}
			P^T \\
			Q^T
		\end{bmatrix}
		v = 0.
	\end{equation*}
	Then we can obtain the solution space of equation~\eqref{eq:linear1}:
	\begin{equation*}
		v = Q \beta = \sum_{i=1}^D q_i \beta_i,
	\end{equation*}
	where $\{q_i\}_{i=1}^D$ are the column vectors of matrix $Q$, and $\{\beta_i\}_{i=1}^D$ are their corresponding coefficients. In other words, the column vectors of matrix $Q$ form the bases of the solution space:
	\begin{equation*}
		\begin{bmatrix}
			\mathrm{vec}(\mathrm{d} \rho_{\mathcal{X}}(A_i)) \\
			\mathrm{vec}(\mathrm{d} \rho_{\mathcal{Y}}(A_i))
		\end{bmatrix}
		= v_i = q_i, \quad i \in \{1, 2, \dots, D\}.
	\end{equation*}
	In general, the number of zero singular values of the coefficient matrix $C(\mathcal{D})$ in equation~\eqref{eq:linear1} matches the dimension of the Lie algebra space, and the corresponding right singular vectors are the vector expansion forms of the Lie algebra bases.
	
	We are concerned that the discretization of equation~\eqref{eq:concat1} may lose information, which could lead to the solution of redundant and incorrect Lie algebra bases. In equation~\eqref{eq:linear1}, note that $v \in \mathbb{R}^{n^2 + m^2}$. If the row rank of $C(\mathcal{D})$ is $r$, then the dimension of the solution space $v$ is $n^2 + m^2 - r$. Therefore, as long as we ensure that the row rank of $C(\mathcal{D})$ reaches its maximum, that is $n^2 + m^2 - D$, we can completely obtain the correct $D$ Lie algebra bases. This assumption can be formalized as $\forall x \in \mathcal{X}: \mathrm{rank}(C(\mathcal{D} \cup \{x\})) = \mathrm{rank}(C(\mathcal{D}))$.
\end{proof}

\subsection{Proof of Theorem~\ref{thm:multi}}
\label{sec:proof-multi}

\multi*

\begin{proof}
	We can express equation~\eqref{eq:measure2} as:
	\begin{align*}
		& \left[ \bigoplus_{i=1}^{c_y} \mathrm{d} \rho_{\mathcal{Y}_i}(A) \right]
		\begin{bmatrix}
			f_1(x) \\
			f_2(x) \\
			\vdots \\
			f_{c_y}(x)
		\end{bmatrix}
		= \\ &
		\begin{bmatrix}
			\nabla f_1(x_1) & \nabla f_1(x_2) & \cdots & \nabla f_1(x_{c_x}) \\
			\nabla f_2(x_1) & \nabla f_2(x_2) & \cdots & \nabla f_2(x_{c_x}) \\
			\vdots & \vdots & \ddots & \vdots \\
			\nabla f_{c_y}(x_1) & \nabla f_{c_y}(x_2) & \cdots & \nabla f_{c_y}(x_{c_x}) \\
		\end{bmatrix}
		\left[ \bigoplus_{i=1}^{c_x} \mathrm{d} \rho_{\mathcal{X}_i}(A) \right]
		\begin{bmatrix}
			x_1 \\
			x_2 \\
			\vdots \\
			x_{c_x}
		\end{bmatrix}.
	\end{align*}
	In other words:
	\begin{equation*}
		\forall i \in \{1, 2, \dots, c_y\}, x \in \mathcal{X}: \quad \mathrm{d} \rho_{\mathcal{Y}_i}(A) f_i(x) = \sum_{j=1}^{c_x} \nabla f_i(x_j) \mathrm{d} \rho_{\mathcal{X}_j}(A) x_j.
	\end{equation*}
	Similar to the single-channel case, we flatten the Lie algebra representation from matrices to vectors:
	\begin{align}
		\label{eq:concat2}
		& \forall i \in \{1, 2, \dots, c_y\}, x \in \mathcal{X}: \nonumber \\
		& [I_{m_i} \otimes f_i^T(x)] \mathrm{vec}(\mathrm{d} \rho_{\mathcal{Y}_i}(A)) = \sum_{j=1}^{c_x} [\nabla f_i(x_j) \otimes x_j^T] \mathrm{vec}(\mathrm{d} \rho_{\mathcal{X}_j}(A)).
	\end{align}
	We discretize equation~\eqref{eq:concat2} using $N$ data samples and obtain a system of linear equations:
	\begin{equation}
		\label{eq:linear2}
		C(\mathcal{D}) v = 
		\begin{bmatrix}
			C_x^{(1)} & C_y^{(1)} \\
			C_x^{(2)} & C_y^{(2)} \\
			\vdots & \vdots \\
			C_x^{(N)} & C_y^{(N)}
		\end{bmatrix}
		\begin{bmatrix}
			v_x \\
			v_y
		\end{bmatrix}
		=0,
	\end{equation}
	where
	\begin{equation*}
		\begin{cases}
			C_x^{(i)} =
			\begin{bmatrix}
				- \nabla f_1(x_1^{(i)}) \otimes x_1^{(i) T} & - \nabla f_1(x_2^{(i)}) \otimes x_2^{(i) T} & \cdots & - \nabla f_1(x_{c_x}^{(i)}) \otimes x_{c_x}^{(i) T} \\
				- \nabla f_2(x_1^{(i)}) \otimes x_1^{(i) T} & - \nabla f_2(x_2^{(i)}) \otimes x_2^{(i) T} & \cdots & - \nabla f_2(x_{c_x}^{(i)}) \otimes x_{c_x}^{(i) T} \\
				\vdots & \vdots & \ddots & \vdots \\
				- \nabla f_{c_y}(x_1^{(i)}) \otimes x_1^{(i) T} & - \nabla f_{c_y}(x_2^{(i)}) \otimes x_2^{(i) T} & \cdots & - \nabla f_{c_y}(x_{c_x}^{(i)}) \otimes x_{c_x}^{(i) T} \\
			\end{bmatrix}, \\
			C_y^{(i)} = \mathrm{diag}[I_{m_1} \otimes f_1^T(x^{(i)}), I_{m_2} \otimes f_2^T(x^{(i)}), \dots, I_{m_{c_y}} \otimes f_{c_y}^T(x^{(i)})], \\
			v_x =
			\begin{bmatrix}
				\mathrm{vec}(\mathrm{d} \rho_{\mathcal{X}_1}(A)) \\
				\mathrm{vec}(\mathrm{d} \rho_{\mathcal{X}_2}(A)) \\
				\vdots \\
				\mathrm{vec}(\mathrm{d} \rho_{\mathcal{X}_{c_x}}(A))
			\end{bmatrix}, \quad
			v_y = 
			\begin{bmatrix}
				\mathrm{vec}(\mathrm{d} \rho_{\mathcal{Y}_1}(A)) \\
				\mathrm{vec}(\mathrm{d} \rho_{\mathcal{Y}_2}(A)) \\
				\vdots \\
				\mathrm{vec}(\mathrm{d} \rho_{\mathcal{Y}_{c_y}}(A))
			\end{bmatrix}.
		\end{cases}
	\end{equation*}
	Then we use the same method as in the single-channel case to solve equation~\eqref{eq:linear2} and obtain the Lie algebra representation for each channel.
\end{proof}

\subsection{Proof of Theorem~\ref{thm:tensor}}
\label{sec:proof-tensor}

\tensor*

\begin{proof}
	Equation~\eqref{eq:measure2} can be written as:
	\begin{equation}
		\label{eq:concat3}
		[\mathrm{d} \rho_{\mathcal{Y}_1}(A) \otimes I_{m_2} + I_{m_1} \otimes \mathrm{d} \rho_{\mathcal{Y}_2}(A)] f(x) = \nabla f(x) [\mathrm{d} \rho_{\mathcal{X}_1} \otimes I_{n_2} + I_{n_1} \otimes \mathrm{d} \rho_{\mathcal{X}_2}] x,
	\end{equation}
	where $x \in \mathbb{R}^{(n_1 \times n_2) \times 1}$ is the vectorized form of the input matrix $X \in \mathcal{X}$. To rewrite equation~\eqref{eq:concat3} in a form similar to equation~\eqref{eq:measure2}, we obtain:
	\begin{align*}
		& \mathrm{vec}(\mathrm{d} \rho_{\mathcal{Y}_1}(A) F(X)) + 
		\begin{bmatrix}
			\mathrm{d} \rho_{\mathcal{Y}_2}(A) F^T_{1 \cdot}(X) \\
			\mathrm{d} \rho_{\mathcal{Y}_2}(A) F^T_{2 \cdot}(X) \\
			\vdots \\
			\mathrm{d} \rho_{\mathcal{Y}_2}(A) F^T_{m_1 \cdot}(X)
		\end{bmatrix}
		= \\
		& \sum_{k=1}^{n_2} \nabla f(X_{\cdot k}) \mathrm{d} \rho_{\mathcal{X}_1}(A) X_{\cdot k} + \sum_{k=1}^{n_1} \nabla f(X_{k \cdot}) \mathrm{d} \rho_{\mathcal{X}_2}(A) X_{k \cdot}^T,
	\end{align*}
	Similar to equation~\eqref{eq:linear0}, we construct a system of linear equations:
	\begin{equation}
		\label{eq:linear3}
		C(\mathcal{D}) v =
		\begin{bmatrix}
			C_{x_1}^{(1)} & C_{x_2}^{(1)} & C_{y_1}^{(1)} & C_{y_2}^{(1)} \\
			C_{x_1}^{(2)} & C_{x_2}^{(2)} & C_{y_1}^{(2)} & C_{y_2}^{(2)} \\
			\vdots & \vdots & \vdots & \vdots \\
			C_{x_1}^{(N)} & C_{x_2}^{(N)} & C_{y_1}^{(N)} & C_{y_2}^{(N)}
		\end{bmatrix}
		\begin{bmatrix}
			v_{x_1} \\
			v_{x_2} \\
			v_{y_1} \\
			v_{y_2}
		\end{bmatrix}
		= 0,
	\end{equation}
	where
	\begin{equation*}
		\begin{cases}
			C_{x_1}^{(i)} = -\sum_{k=1}^{n_2} \nabla f(X_{\cdot k}^{(i)}) \otimes X_{\cdot k}^{(i) T}, \quad C_{x_2}^{(i)} = -\sum_{k=1}^{n_1} \nabla f(X_{k \cdot}^{(i)}) \otimes X_{k \cdot}^{(i)}, \\
			C_{y_1}^{(i)} = I_{m_1} \otimes F^T(X^{(i)}), \quad C_{y_2}^{(i)} = 
			\begin{bmatrix}
				I_{m_2} \otimes F_{1 \cdot}(X^{(i)}) \\
				I_{m_2} \otimes F_{2 \cdot}(X^{(i)}) \\
				\vdots \\
				I_{m_2} \otimes F_{m_1 \cdot}(X^{(i)}) \\
			\end{bmatrix}, \\
			v_{x_i} = \mathrm{vec}(\mathrm{d} \rho_{\mathcal{X}_i}(A)), \quad
			v_{y_i} = \mathrm{vec}(\mathrm{d} \rho_{\mathcal{Y}_i}(A)).
		\end{cases}
	\end{equation*}
	Therefore, we can obtain the Lie algebra space by solving equation~\eqref{eq:linear3}.
\end{proof}

\section{Implementation details}
\label{sec:implementation}

\subsection{Two-body problem}
\label{sec:implementation-2body}

The two-body problem takes the positions and momentum coordinates $q_1, p_1, q_2, p_2$ of two particles on a plane as input and output. Therefore, we use a 3-layer MLP with input dimension of 8, hidden dimension of 384, output dimension of 8, and Tanh activation function to fit the mapping of the problem. We choose the Adan optimizer \citep{xie2024adan} for training and set the number of epochs to 10. To shuffle the data distribution, for data points where $q_1$ is in the first or third quadrant, we rotate both its input and output by 90 degrees. The size of the training set is 18,000, and we use the entire training set for symmetry discovery. We perform this experiment on a single-core NVIDIA GeForce RTX 3090 GPU with available memory of 24,576 MiB, and the execution time is approximately 3 minutes.

We consider the two-body problem as a multi-channel case, where position and momentum coordinates transform independently. For the equivariance discovery of such sequence prediction problems, we hope that the group representation of the input space is the same as that of the output space. Formally, $\rho_{\mathcal{X}}(g) = \rho_{\mathcal{Y}}(g) = \rho(g)$, and $\mathrm{d} \rho_{\mathcal{X}}(A) = \mathrm{d} \rho_{\mathcal{Y}}(A) = \mathrm{d} \rho(A)$. Setting $v = v_x = v_y$ in equation~\eqref{eq:linear2}, we obtain:
\begin{equation*}
	C(\mathcal{D}) v = 
	\begin{bmatrix}
		C_x^{(1)} + C_y^{(1)} \\
		C_x^{(2)} + C_y^{(2)} \\
		\vdots \\
		C_x^{(N)} + C_y^{(N)}
	\end{bmatrix}
	v = 0.
\end{equation*}
Then, $C(\mathcal{D})^T C(\mathcal{D})$ can be computed as follows:
\begin{equation*}
	C(\mathcal{D})^T C(\mathcal{D}) = \sum_{i=1}^N (C_x^{(i)} + C_y^{(i)})^T (C_x^{(i)} + C_y^{(i)}).
\end{equation*}

When conducting quantitative error analysis, we repeat experiments with five different random seeds. The random factor for \model~is the parameter initialization of the neural network, while for LieGAN, it's the parameter initialization of the generator and discriminator. We report the average values of space error and orthogonality error in Table~\ref{tab:2body} (in Section~\ref{sec:experiments-2body}), and indicate their minimum and maximum values in gray font.

\subsection{The moment of inertia matrix prediction}
\label{sec:implementation-inertia}

For convenience, we assume that the mass of all particles on the rigid body is 1. Therefore, the problem takes the spatial coordinates of several particles $x_i \in \mathbb{R}^3$ as input and outputs the inertia tensor of the rigid body $M \in \mathbb{R}^{3 \times 3}$. In practice, we set the number of particles to 3. The input dimension of the three-layer MLP used for training is 9, the hidden dimension is 384, and the output dimension is 9, with the ReLU function used as the activation function. We train for 100 epochs using the Adan optimizer \citep{xie2024adan}. The size of the training set is 100,000, and we use $\frac{1}{10}$ of it for symmetry discovery. We perform this experiment on a single-core NVIDIA GeForce RTX 3090 GPU with available memory of 24,576 MiB, and the execution time is approximately 20 minutes.

For this problem, the input is a multi-channel vector $\mathcal{X} = \bigoplus_i \mathcal{X}_i$, and the output is a matrix $\mathcal{Y} = \mathcal{Y}_1 \otimes \mathcal{Y}_2$. Furthermore, different spatial coordinates of particles should share the same group transformation. Formally, $\rho_{\mathcal{X}_1}(g) = \dots = \rho_{\mathcal{X}_{c_x}}(g) = \rho_{\mathcal{X}}(g)$, and $\mathrm{d} \rho_{\mathcal{X}_1}(A) = \dots = \mathrm{d} \rho_{\mathcal{X}_{c_x}}(A) = \mathrm{d} \rho_{\mathcal{X}}(A)$. Combining equation~\eqref{eq:linear2} and equation~\eqref{eq:linear3}, we obtain:
\begin{equation*}
	C(\mathcal{D}) v = 
	\begin{bmatrix}
		C_x^{(1)} & C_{y_1}^{(1)}  & C_{y_2}^{(1)} \\
		C_x^{(2)} & C_{y_1}^{(2)}  & C_{y_2}^{(2)} \\
		\vdots & \vdots & \vdots \\
		C_x^{(N)} & C_{y_1}^{(N)}  & C_{y_2}^{(N)}
	\end{bmatrix}
	\begin{bmatrix}
		v_x \\
		v_{y_1} \\
		v_{y_2}
	\end{bmatrix}
	= 0,
\end{equation*}
where
\begin{equation*}
	\begin{cases}
		C_x^{(i)} = - \sum_{j=1}^{c_x} \nabla f(x_j^{(i)}) \otimes x_j^{(i) T}, \\
		C_{y_1}^{(i)} = I_{m_1} \otimes F^T(x^{(i)}), \quad C_{y_2}^{(i)} = 
		\begin{bmatrix}
			I_{m_2} \otimes F_{1 \cdot}(x^{(i)}) \\
			I_{m_2} \otimes F_{2 \cdot}(x^{(i)}) \\
			\vdots \\
			I_{m_2} \otimes F_{m_1 \cdot}(x^{(i)}) \\
		\end{bmatrix}, \\
		v_{x} = \mathrm{vec}(\mathrm{d} \rho_{\mathcal{X}}(A)), \quad
		v_{y_i} = \mathrm{vec}(\mathrm{d} \rho_{\mathcal{Y}_i}(A)).
	\end{cases}
\end{equation*}
Similarly, we compute $C^T C$ to save memory overhead:
\begin{equation*}
	C(\mathcal{D})^T C(\mathcal{D}) = 
	\begin{bmatrix}
		\sum_{i=1}^N C_x^{(i) T} C_x^{(i)} & \sum_{i=1}^N C_x^{(i) T} C_{y_1}^{(i)} & \sum_{i=1}^N C_x^{(i) T} C_{y_2}^{(i)} \\
		\sum_{i=1}^N C_{y_1}^{(i) T} C_x^{(i)} & \sum_{i=1}^N C_{y_1}^{(i) T} C_{y_1}^{(i)} & \sum_{i=1}^N C_{y_1}^{(i) T} C_{y_2}^{(i)} \\
		\sum_{i=1}^N C_{y_2}^{(i) T} C_x^{(i)} & \sum_{i=1}^N C_{y_2}^{(i) T} C_{y_1}^{(i)} & \sum_{i=1}^N C_{y_2}^{(i) T} C_{y_2}^{(i)}
	\end{bmatrix}.
\end{equation*}
Then, the real Lie algebra bases are $\{\mathrm{d} \rho_{\mathcal{X}_i}(A_j)\}_j = \{\mathrm{d} \rho_{\mathcal{Y}_1}(A_j)\}_j = \{\mathrm{d} \rho_{\mathcal{Y}_2}(A_j)\}_j = \{
\begin{bmatrix}
	0 & 1 & 0 \\
	-1 & 0 & 0 \\
	0 & 0 & 0
\end{bmatrix},
\begin{bmatrix}
	0 & 0 & 1 \\
	0 & 0 & 0 \\
	-1 & 0 & 0
\end{bmatrix},
\begin{bmatrix}
	0 & 0 & 0 \\
	0 & 0 & 1 \\
	0 & -1 & 0
\end{bmatrix},
\begin{bmatrix}
	1 & 0 & 0 \\
	0 & 1 & 0 \\
	0 & 0 & 1
\end{bmatrix}
\}$.
In the ablation study, we process the output as a single-channel vector, which means the group representation on $\mathcal{Y}$ is a $9 \times 9$ matrix $\rho_{\mathcal{Y}}(g) \in \mathbb{R}^{9 \times 9}$ instead of decomposing it into two $3 \times 3$ matrices $\rho_{\mathcal{Y}_1}(g) \in \mathbb{R}^{3 \times 3}$ and $\rho_{\mathcal{Y}_2}(g) \in \mathbb{R}^{3 \times 3}$. Then we transform equation~\eqref{eq:linear2} into:
\begin{equation*}
	C(\mathcal{D}) v = 
	\begin{bmatrix}
		- \sum_{j=1}^{c_x} \nabla f(x_j^{(1)}) \otimes x_j^{(1) T} & I_m \otimes f^T(x^{(1)}) \\
		- \sum_{j=1}^{c_x} \nabla f(x_j^{(2)}) \otimes x_j^{(2) T} & I_m \otimes f^T(x^{(2)}) \\
		\vdots & \vdots \\
		- \sum_{j=1}^{c_x} \nabla f(x_j^{(N)}) \otimes x_j^{(N) T} & I_m \otimes f^T(x^{(N)}) \\
	\end{bmatrix}
	\begin{bmatrix}
		\mathrm{vec}(\mathrm{d} \rho_{\mathcal{X}}(A)) \\
		\mathrm{vec}(\mathrm{d} \rho_{\mathcal{Y}}(A))
	\end{bmatrix}
	= 0.
\end{equation*}

For quantitative error analysis, similar to Table~\ref{tab:2body} (in Section~\ref{sec:experiments-2body}), we conduct five repeated experiments with different random seeds and report the average, minimum, and maximum values in Table~\ref{tab:inertia} (in Section~\ref{sec:experiments-inertia}).

\subsection{Top quark tagging}
\label{sec:implementation-top-quark-tagging}

Top quark tagging takes the four-momenta $p_i^\mu \in \mathbb{R}^4$ of 20 jet constituents as input and predicts labels for the particles as output. We set the input dimension of a 3-layer MLP to 80, the hidden dimension to 200, and the output dimension to 1, choosing the ReLU function as the activation function. We use the Adan optimizer \citep{xie2024adan} for training, with the number of epochs set to 100. The size of the training set is 1,211,000, and we use $\frac{1}{10}$ of the training set for symmetry discovery. We perform this experiment on a single-core NVIDIA GeForce RTX 3090 GPU with available memory of 24,576 MiB, and the execution time is approximately 2 hours.

For such problems with invariance, we strictly require $\rho_{\mathcal{Y}}(g) = I_m$, i.e., $\mathrm{d} \rho_{\mathcal{Y}}(A) = 0$. Additionally, the four-momentum of different jet constituents should share the same group representation. Setting $v_y = 0$ and $\mathrm{d} \rho_{\mathcal{X}_1}(A) = \dots = \mathrm{d} \rho_{\mathcal{X}_{c_x}}(A) = \mathrm{d} \rho_{\mathcal{X}}(A)$ in equation~\eqref{eq:linear2}, we derive the system of linear equations:
\begin{equation*}
	C(\mathcal{D}) v = 
	\begin{bmatrix}
		- \sum_{j=1}^{c_x} \nabla f(x_j^{(1)}) \otimes x_j^{(1) T}\\
		- \sum_{j=1}^{c_x} \nabla f(x_j^{(2)}) \otimes x_j^{(2) T}\\
		\vdots\\
		- \sum_{j=1}^{c_x} \nabla f(x_j^{(N)}) \otimes x_j^{(N) T}\\
	\end{bmatrix}
	\mathrm{vec}(\mathrm{d} \rho_{\mathcal{X}}(A)) = 0.
\end{equation*}
The computation of $C(\mathcal{D})^T C(\mathcal{D})$ can be done as follows:
\begin{equation*}
	C(\mathcal{D})^T C(\mathcal{D}) = \sum_{i=1}^N \sum_{j=1}^{c_x} \nabla f(x_j^{(i)})^T \nabla f(x_j^{(i)}) \otimes x_j^{(i)} x_j^{(i) T}.
\end{equation*}

Then we conduct quantitative error analysis. Five different random seeds are used for \model, and three for LieGAN. We report the average, minimum, and maximum values of space error and orthogonality error in Table~\ref{tab:top-quark-tagging} (in Section~\ref{sec:experiments-top-quark-tagging}).




\bibliographystyle{elsarticle-harv}
\bibliography{references}

\begin{thebibliography}{33}
\expandafter\ifx\csname natexlab\endcsname\relax\def\natexlab#1{#1}\fi
\providecommand{\url}[1]{\texttt{#1}}
\providecommand{\href}[2]{#2}
\providecommand{\path}[1]{#1}
\providecommand{\DOIprefix}{doi:}
\providecommand{\ArXivprefix}{arXiv:}
\providecommand{\URLprefix}{URL: }
\providecommand{\Pubmedprefix}{pmid:}
\providecommand{\doi}[1]{\href{http://dx.doi.org/#1}{\path{#1}}}
\providecommand{\Pubmed}[1]{\href{pmid:#1}{\path{#1}}}
\providecommand{\bibinfo}[2]{#2}
\ifx\xfnm\relax \def\xfnm[#1]{\unskip,\space#1}\fi
%Type = Article
\bibitem[{Allingham et~al.(2024)Allingham, Mlodozeniec, Padhy, Antor{\'a}n,
  Krueger, Turner, Nalisnick and
  Hern{\'a}ndez-Lobato}]{allingham2024generative}
\bibinfo{author}{Allingham, J.U.}, \bibinfo{author}{Mlodozeniec, B.K.},
  \bibinfo{author}{Padhy, S.}, \bibinfo{author}{Antor{\'a}n, J.},
  \bibinfo{author}{Krueger, D.}, \bibinfo{author}{Turner, R.E.},
  \bibinfo{author}{Nalisnick, E.}, \bibinfo{author}{Hern{\'a}ndez-Lobato,
  J.M.}, \bibinfo{year}{2024}.
\newblock \bibinfo{title}{A generative model of symmetry transformations}.
\newblock \bibinfo{journal}{arXiv preprint arXiv:2403.01946} .
%Type = Article
\bibitem[{Benton et~al.(2020)Benton, Finzi, Izmailov and
  Wilson}]{benton2020learning}
\bibinfo{author}{Benton, G.}, \bibinfo{author}{Finzi, M.},
  \bibinfo{author}{Izmailov, P.}, \bibinfo{author}{Wilson, A.G.},
  \bibinfo{year}{2020}.
\newblock \bibinfo{title}{Learning invariances in neural networks from training
  data}.
\newblock \bibinfo{journal}{Advances in Neural Information Processing Systems}
  \bibinfo{volume}{33}, \bibinfo{pages}{17605--17616}.
%Type = Inproceedings
\bibitem[{Cohen and Welling(2016a)}]{cohen2016group}
\bibinfo{author}{Cohen, T.}, \bibinfo{author}{Welling, M.},
  \bibinfo{year}{2016}a.
\newblock \bibinfo{title}{Group equivariant convolutional networks}, in:
  \bibinfo{booktitle}{International Conference on Machine Learning},
  \bibinfo{organization}{PMLR}. pp. \bibinfo{pages}{2990--2999}.
%Type = Article
\bibitem[{Cohen and Welling(2016b)}]{cohen2016steerable}
\bibinfo{author}{Cohen, T.S.}, \bibinfo{author}{Welling, M.},
  \bibinfo{year}{2016}b.
\newblock \bibinfo{title}{Steerable {CNN}s}.
\newblock \bibinfo{journal}{arXiv preprint arXiv:1612.08498} .
%Type = Article
\bibitem[{Dehmamy et~al.(2021)Dehmamy, Walters, Liu, Wang and
  Yu}]{dehmamy2021automatic}
\bibinfo{author}{Dehmamy, N.}, \bibinfo{author}{Walters, R.},
  \bibinfo{author}{Liu, Y.}, \bibinfo{author}{Wang, D.}, \bibinfo{author}{Yu,
  R.}, \bibinfo{year}{2021}.
\newblock \bibinfo{title}{Automatic symmetry discovery with {L}ie algebra
  convolutional network}.
\newblock \bibinfo{journal}{Advances in Neural Information Processing Systems}
  \bibinfo{volume}{34}, \bibinfo{pages}{2503--2515}.
%Type = Article
\bibitem[{Desai et~al.(2022)Desai, Nachman and Thaler}]{desai2022symmetry}
\bibinfo{author}{Desai, K.}, \bibinfo{author}{Nachman, B.},
  \bibinfo{author}{Thaler, J.}, \bibinfo{year}{2022}.
\newblock \bibinfo{title}{Symmetry discovery with deep learning}.
\newblock \bibinfo{journal}{Physical Review D} \bibinfo{volume}{105},
  \bibinfo{pages}{096031}.
%Type = Incollection
\bibitem[{Eilertsen et~al.(2020)Eilertsen, J{\"o}nsson, Ropinski, Unger and
  Ynnerman}]{eilertsen2020classifying}
\bibinfo{author}{Eilertsen, G.}, \bibinfo{author}{J{\"o}nsson, D.},
  \bibinfo{author}{Ropinski, T.}, \bibinfo{author}{Unger, J.},
  \bibinfo{author}{Ynnerman, A.}, \bibinfo{year}{2020}.
\newblock \bibinfo{title}{Classifying the classifier: dissecting the weight
  space of neural networks}, in: \bibinfo{booktitle}{ECAI 2020}.
  \bibinfo{publisher}{IOS Press}, pp. \bibinfo{pages}{1119--1126}.
%Type = Inproceedings
\bibitem[{Finzi et~al.(2021)Finzi, Welling and Wilson}]{finzi2021practical}
\bibinfo{author}{Finzi, M.}, \bibinfo{author}{Welling, M.},
  \bibinfo{author}{Wilson, A.G.}, \bibinfo{year}{2021}.
\newblock \bibinfo{title}{A practical method for constructing equivariant
  multilayer perceptrons for arbitrary matrix groups}, in:
  \bibinfo{booktitle}{International Conference on Machine Learning},
  \bibinfo{organization}{PMLR}. pp. \bibinfo{pages}{3318--3328}.
%Type = Article
\bibitem[{Greydanus et~al.(2019)Greydanus, Dzamba and
  Yosinski}]{greydanus2019hamiltonian}
\bibinfo{author}{Greydanus, S.}, \bibinfo{author}{Dzamba, M.},
  \bibinfo{author}{Yosinski, J.}, \bibinfo{year}{2019}.
\newblock \bibinfo{title}{Hamiltonian neural networks}.
\newblock \bibinfo{journal}{Advances in Neural Information Processing Systems}
  \bibinfo{volume}{32}.
%Type = Inproceedings
\bibitem[{He et~al.(2022)He, Chen, Shen, Yang and Lin}]{he2022neural}
\bibinfo{author}{He, L.}, \bibinfo{author}{Chen, Y.}, \bibinfo{author}{Shen,
  Z.}, \bibinfo{author}{Yang, Y.}, \bibinfo{author}{Lin, Z.},
  \bibinfo{year}{2022}.
\newblock \bibinfo{title}{Neural e{PDO}s: Spatially adaptive equivariant
  partial differential operator based networks}, in: \bibinfo{booktitle}{The
  Eleventh International Conference on Learning Representations}.
%Type = Article
\bibitem[{Kipf and Welling(2016)}]{kipf2016semi}
\bibinfo{author}{Kipf, T.N.}, \bibinfo{author}{Welling, M.},
  \bibinfo{year}{2016}.
\newblock \bibinfo{title}{Semi-supervised classification with graph
  convolutional networks}.
\newblock \bibinfo{journal}{arXiv preprint arXiv:1609.02907} .
%Type = Article
\bibitem[{Krippendorf and Syvaeri(2020)}]{krippendorf2020detecting}
\bibinfo{author}{Krippendorf, S.}, \bibinfo{author}{Syvaeri, M.},
  \bibinfo{year}{2020}.
\newblock \bibinfo{title}{Detecting symmetries with neural networks}.
\newblock \bibinfo{journal}{Machine Learning: Science and Technology}
  \bibinfo{volume}{2}, \bibinfo{pages}{015010}.
%Type = Article
\bibitem[{LeCun et~al.(1998)LeCun, Bottou, Bengio and
  Haffner}]{lecun1998gradient}
\bibinfo{author}{LeCun, Y.}, \bibinfo{author}{Bottou, L.},
  \bibinfo{author}{Bengio, Y.}, \bibinfo{author}{Haffner, P.},
  \bibinfo{year}{1998}.
\newblock \bibinfo{title}{Gradient-based learning applied to document
  recognition}.
\newblock \bibinfo{journal}{Proceedings of the IEEE} \bibinfo{volume}{86},
  \bibinfo{pages}{2278--2324}.
%Type = Inproceedings
\bibitem[{Li et~al.(2024)Li, Qiu, Chen, He and Lin}]{li2024affine}
\bibinfo{author}{Li, Y.}, \bibinfo{author}{Qiu, Y.}, \bibinfo{author}{Chen,
  Y.}, \bibinfo{author}{He, L.}, \bibinfo{author}{Lin, Z.},
  \bibinfo{year}{2024}.
\newblock \bibinfo{title}{Affine equivariant networks based on differential
  invariants}, in: \bibinfo{booktitle}{Proceedings of the IEEE/CVF Conference
  on Computer Vision and Pattern Recognition}, pp. \bibinfo{pages}{5546--5556}.
%Type = Inproceedings
\bibitem[{MacDonald et~al.(2022)MacDonald, Ramasinghe and
  Lucey}]{macdonald2022enabling}
\bibinfo{author}{MacDonald, L.E.}, \bibinfo{author}{Ramasinghe, S.},
  \bibinfo{author}{Lucey, S.}, \bibinfo{year}{2022}.
\newblock \bibinfo{title}{Enabling equivariance for arbitrary {L}ie groups},
  in: \bibinfo{booktitle}{Proceedings of the IEEE/CVF Conference on Computer
  Vision and Pattern Recognition}, pp. \bibinfo{pages}{8183--8192}.
%Type = Article
\bibitem[{Moskalev et~al.(2022)Moskalev, Sepliarskaia, Sosnovik and
  Smeulders}]{moskalev2022liegg}
\bibinfo{author}{Moskalev, A.}, \bibinfo{author}{Sepliarskaia, A.},
  \bibinfo{author}{Sosnovik, I.}, \bibinfo{author}{Smeulders, A.},
  \bibinfo{year}{2022}.
\newblock \bibinfo{title}{Lie{GG}: Studying learned {L}ie group generators}.
\newblock \bibinfo{journal}{Advances in Neural Information Processing Systems}
  \bibinfo{volume}{35}, \bibinfo{pages}{25212--25223}.
%Type = Inproceedings
\bibitem[{Navon et~al.(2023)Navon, Shamsian, Achituve, Fetaya, Chechik and
  Maron}]{navon2023equivariant}
\bibinfo{author}{Navon, A.}, \bibinfo{author}{Shamsian, A.},
  \bibinfo{author}{Achituve, I.}, \bibinfo{author}{Fetaya, E.},
  \bibinfo{author}{Chechik, G.}, \bibinfo{author}{Maron, H.},
  \bibinfo{year}{2023}.
\newblock \bibinfo{title}{Equivariant architectures for learning in deep weight
  spaces}, in: \bibinfo{booktitle}{International Conference on Machine
  Learning}, \bibinfo{organization}{PMLR}. pp. \bibinfo{pages}{25790--25816}.
%Type = Article
\bibitem[{van~der Ouderaa et~al.(2024)van~der Ouderaa, Immer and van~der
  Wilk}]{van2024learning}
\bibinfo{author}{van~der Ouderaa, T.}, \bibinfo{author}{Immer, A.},
  \bibinfo{author}{van~der Wilk, M.}, \bibinfo{year}{2024}.
\newblock \bibinfo{title}{Learning layer-wise equivariances automatically using
  gradients}.
\newblock \bibinfo{journal}{Advances in Neural Information Processing Systems}
  \bibinfo{volume}{36}.
%Type = Article
\bibitem[{Romero and Lohit(2022)}]{romero2022learning}
\bibinfo{author}{Romero, D.W.}, \bibinfo{author}{Lohit, S.},
  \bibinfo{year}{2022}.
\newblock \bibinfo{title}{Learning partial equivariances from data}.
\newblock \bibinfo{journal}{Advances in Neural Information Processing Systems}
  \bibinfo{volume}{35}, \bibinfo{pages}{36466--36478}.
%Type = Inproceedings
\bibitem[{Satorras et~al.(2021)Satorras, Hoogeboom and Welling}]{satorras2021n}
\bibinfo{author}{Satorras, V.G.}, \bibinfo{author}{Hoogeboom, E.},
  \bibinfo{author}{Welling, M.}, \bibinfo{year}{2021}.
\newblock \bibinfo{title}{E (n) equivariant graph neural networks}, in:
  \bibinfo{booktitle}{International Conference on Machine Learning},
  \bibinfo{organization}{PMLR}. pp. \bibinfo{pages}{9323--9332}.
%Type = Article
\bibitem[{Sch{\"u}rholt et~al.(2021)Sch{\"u}rholt, Kostadinov and
  Borth}]{schurholt2021self}
\bibinfo{author}{Sch{\"u}rholt, K.}, \bibinfo{author}{Kostadinov, D.},
  \bibinfo{author}{Borth, D.}, \bibinfo{year}{2021}.
\newblock \bibinfo{title}{Self-supervised representation learning on neural
  network weights for model characteristic prediction}.
\newblock \bibinfo{journal}{Advances in Neural Information Processing Systems}
  \bibinfo{volume}{34}, \bibinfo{pages}{16481--16493}.
%Type = Inproceedings
\bibitem[{Shen et~al.(2020)Shen, He, Lin and Ma}]{shen2020pdo}
\bibinfo{author}{Shen, Z.}, \bibinfo{author}{He, L.}, \bibinfo{author}{Lin,
  Z.}, \bibinfo{author}{Ma, J.}, \bibinfo{year}{2020}.
\newblock \bibinfo{title}{{PDO}-e{C}onvs: Partial differential operator based
  equivariant convolutions}, in: \bibinfo{booktitle}{International Conference
  on Machine Learning}, \bibinfo{organization}{PMLR}. pp.
  \bibinfo{pages}{8697--8706}.
%Type = Inproceedings
\bibitem[{Shen et~al.(2022)Shen, Hong, She, Ma and Lin}]{shen2022pdo}
\bibinfo{author}{Shen, Z.}, \bibinfo{author}{Hong, T.}, \bibinfo{author}{She,
  Q.}, \bibinfo{author}{Ma, J.}, \bibinfo{author}{Lin, Z.},
  \bibinfo{year}{2022}.
\newblock \bibinfo{title}{{PDO}-s3{DCNN}s: Partial differential operator based
  steerable 3{D} {CNN}s}, in: \bibinfo{booktitle}{International Conference on
  Machine Learning}, \bibinfo{organization}{PMLR}. pp.
  \bibinfo{pages}{19827--19846}.
%Type = Inproceedings
\bibitem[{Shen et~al.(2021)Shen, Shen, Lin and Ma}]{shen2021pdo}
\bibinfo{author}{Shen, Z.}, \bibinfo{author}{Shen, T.}, \bibinfo{author}{Lin,
  Z.}, \bibinfo{author}{Ma, J.}, \bibinfo{year}{2021}.
\newblock \bibinfo{title}{{PDO}-e{S2CNN}s: Partial differential operator based
  equivariant spherical {CNN}s}, in: \bibinfo{booktitle}{Proceedings of the
  AAAI Conference on Artificial Intelligence}, pp. \bibinfo{pages}{9585--9593}.
%Type = Inproceedings
\bibitem[{Tegn{\'e}r and Kjellstrom(2023)}]{tegner2023self}
\bibinfo{author}{Tegn{\'e}r, G.}, \bibinfo{author}{Kjellstrom, H.},
  \bibinfo{year}{2023}.
\newblock \bibinfo{title}{Self-supervised latent symmetry discovery via
  class-pose decomposition}, in: \bibinfo{booktitle}{NeurIPS 2023 Workshop on
  Symmetry and Geometry in Neural Representations}.
%Type = Article
\bibitem[{Unterthiner et~al.(2020)Unterthiner, Keysers, Gelly, Bousquet and
  Tolstikhin}]{unterthiner2020predicting}
\bibinfo{author}{Unterthiner, T.}, \bibinfo{author}{Keysers, D.},
  \bibinfo{author}{Gelly, S.}, \bibinfo{author}{Bousquet, O.},
  \bibinfo{author}{Tolstikhin, I.}, \bibinfo{year}{2020}.
\newblock \bibinfo{title}{Predicting neural network accuracy from weights}.
\newblock \bibinfo{journal}{arXiv preprint arXiv:2002.11448} .
%Type = Article
\bibitem[{Weiler and Cesa(2019)}]{weiler2019general}
\bibinfo{author}{Weiler, M.}, \bibinfo{author}{Cesa, G.}, \bibinfo{year}{2019}.
\newblock \bibinfo{title}{General {E} (2)-equivariant steerable {CNN}s}.
\newblock \bibinfo{journal}{Advances in Neural Information Processing Systems}
  \bibinfo{volume}{32}.
%Type = Article
\bibitem[{Weiler et~al.(2018a)Weiler, Geiger, Welling, Boomsma and
  Cohen}]{weiler20183d}
\bibinfo{author}{Weiler, M.}, \bibinfo{author}{Geiger, M.},
  \bibinfo{author}{Welling, M.}, \bibinfo{author}{Boomsma, W.},
  \bibinfo{author}{Cohen, T.S.}, \bibinfo{year}{2018}a.
\newblock \bibinfo{title}{3{D} steerable {CNN}s: Learning rotationally
  equivariant features in volumetric data}.
\newblock \bibinfo{journal}{Advances in Neural Information Processing Systems}
  \bibinfo{volume}{31}.
%Type = Inproceedings
\bibitem[{Weiler et~al.(2018b)Weiler, Hamprecht and
  Storath}]{weiler2018learning}
\bibinfo{author}{Weiler, M.}, \bibinfo{author}{Hamprecht, F.A.},
  \bibinfo{author}{Storath, M.}, \bibinfo{year}{2018}b.
\newblock \bibinfo{title}{Learning steerable filters for rotation equivariant
  {CNN}s}, in: \bibinfo{booktitle}{Proceedings of the IEEE Conference on
  Computer Vision and Pattern Recognition}, pp. \bibinfo{pages}{849--858}.
%Type = Article
\bibitem[{Xie et~al.(2024)Xie, Zhou, Li, Lin and Yan}]{xie2024adan}
\bibinfo{author}{Xie, X.}, \bibinfo{author}{Zhou, P.}, \bibinfo{author}{Li,
  H.}, \bibinfo{author}{Lin, Z.}, \bibinfo{author}{Yan, S.},
  \bibinfo{year}{2024}.
\newblock \bibinfo{title}{Adan: Adaptive nesterov momentum algorithm for faster
  optimizing deep models}.
\newblock \bibinfo{journal}{IEEE Transactions on Pattern Analysis and Machine
  Intelligence} .
%Type = Article
\bibitem[{Yang et~al.(2023a)Yang, Dehmamy, Walters and Yu}]{yang2023latent}
\bibinfo{author}{Yang, J.}, \bibinfo{author}{Dehmamy, N.},
  \bibinfo{author}{Walters, R.}, \bibinfo{author}{Yu, R.},
  \bibinfo{year}{2023}a.
\newblock \bibinfo{title}{Latent space symmetry discovery}.
\newblock \bibinfo{journal}{arXiv preprint arXiv:2310.00105} .
%Type = Inproceedings
\bibitem[{Yang et~al.(2023b)Yang, Walters, Dehmamy and Yu}]{yang2023generative}
\bibinfo{author}{Yang, J.}, \bibinfo{author}{Walters, R.},
  \bibinfo{author}{Dehmamy, N.}, \bibinfo{author}{Yu, R.},
  \bibinfo{year}{2023}b.
\newblock \bibinfo{title}{Generative adversarial symmetry discovery}, in:
  \bibinfo{booktitle}{International Conference on Machine Learning},
  \bibinfo{organization}{PMLR}. pp. \bibinfo{pages}{39488--39508}.
%Type = Article
\bibitem[{Zhou et~al.(2020)Zhou, Knowles and Finn}]{zhou2020meta}
\bibinfo{author}{Zhou, A.}, \bibinfo{author}{Knowles, T.},
  \bibinfo{author}{Finn, C.}, \bibinfo{year}{2020}.
\newblock \bibinfo{title}{Meta-learning symmetries by reparameterization}.
\newblock \bibinfo{journal}{arXiv preprint arXiv:2007.02933} .

\end{thebibliography}
\end{document}